%% file: main.tex
\documentclass{article}

\PassOptionsToPackage{numbers,sort&compress}{natbib}
\usepackage[final,main]{neurips_2026}

\usepackage[utf8]{inputenc}
\usepackage[T1]{fontenc}
\usepackage{amsmath}
\usepackage{amssymb}
\usepackage{amsthm}
\usepackage{booktabs}
\usepackage{enumitem}
\usepackage{graphicx}
\usepackage{microtype}
\usepackage{xcolor}
\usepackage{hyperref}
\usepackage{url}
\usepackage{tikz}
\usetikzlibrary{arrows.meta,positioning}

\title{Loss is Not Behavior: A Unified Output-Space Analysis of Gradient-Based Machine Unlearning}

\author{%
  Xingjian Zhao \quad Mohammad Mohammadi Amiri \quad Malik Magdon-Ismail \\
  Department of Computer Science, Rensselaer Polytechnic Institute, Troy, NY 12180, USA \\
  \texttt{\{zhaox8,mohamm11\}@rpi.edu}, \texttt{magdon@cs.rpi.edu}
}

\newcommand{\D}{\mathcal{D}}
\newcommand{\Dr}{\mathcal{D}_{r}}
\newcommand{\Df}{\mathcal{D}_{f}}
\newcommand{\Robj}{\mathcal{R}}
\newcommand{\Fobj}{\mathcal{F}}
\newcommand{\thetafull}{\boldsymbol{\theta}^{\star}}
\newcommand{\thetaret}{\boldsymbol{\theta}_{r}^{\star}}
\newcommand{\NS}{Newton Step}
\newcommand{\blank}{\textemdash}

\newtheorem{theorem}{Theorem}
\newtheorem{proposition}{Proposition}
\newtheorem{lemma}{Lemma}
\newtheorem{corollary}{Corollary}

\begin{document}

\maketitle

\begin{abstract}
Gradient-based machine unlearning methods typically ``unlearn'' by optimizing forget and retain losses, but similar losses do not necessarily imply similar output-space behavior. 
Prior work has shown that
%loss agreement can mask output-space deviation of unlearned models from retraining.
%Thus, 
high forget loss alone does not certify that unlearning has produced retraining-like output behavior.
We therefore ask: when a gradient-based unlearner differs from the model that would have been learned on only the retain data, what structure does the resulting output-space gap have? 
We derive closed-form expressions for this gap in linear regression, with an extension to nonlinear models via a generalized Gauss-Newton and local linearity approximations. 
Our analysis puts various gradient-based methods under a unified framework parametrized by geometry, reweighting, and projection. 
Two leverage matrices emerge that capture the forget-retain entanglement under the retrained model geometry: forget self-leverage and retain cross-leverage. 
We show both theoretically and empirically that these two leverage matrices drive the impact of the forget residual on the forget and retain outputs, and existing methods account for this geometry only partially.
Our framework yields
%a potential unified comparison of gradient-based unlearners and 
precomputable leverage diagnostics for auditing forget sets to identify what type of gradient-based unlearning is likely to be successful.
\end{abstract}

\section{Introduction}
Modern machine learning increasingly uses sensitive, copyrighted, or safety-critical data. As a result,
machine unlearning, the post-hoc removal of a data's influence on the learned model, has become a practical and regulatory concern \citep{gdpr2016,cao2015towards}. 
The gold standard for machine unlearning is conceptually simple: retrain the model on the retained data after deleting the forget set. 
For large models and frequent deletion requests, however, exact retraining is usually too expensive, so practical unlearning methods seek more efficient approximations to this retrained gold standard model.

One family of approximations uses Newton-style updates to move directly toward the retrained optimum, but these methods depend on Hessian information and are often not scalable to large models \citep{koh2017understanding,guo2019certified,zhao2026win}. 
The dominant practical alternative is gradient-based unlearning: gradient ascent (GA) increases loss on the forget set, while GradDiff, NPO, SimNPO, IMU, and projected variants add retain losses, per-sample weights, or projections to preserve utility \citep{jang2023knowledge,liu2022continual,zhang2024negative,fan2024simplicity,fan2025imu,zhou2025geometric}. 
These methods are usually tuned and compared through forget and retain losses.

Recent work has established that loss-space, parameter-space agreement with retraining, and even finite output-imitation tests are insufficient evaluation criteria for unlearning \citep{lee2025distillation,yu2025impossibility}. 
We verify this empirically in Figure~\ref{fig:loss_not_behavior}: at matched forget loss or similar weights, different gradient-based unlearners can have very different output-space behavior and forget performance, and they can all differ substantially from the retrained model.
This motivates output-space deviation from retraining as the evaluation surface, which leads to the open question we study here: what governs the structure of this deviation for gradient-based unlearners?
We show that the answer is a leverage geometry induced by the retrained model, and existing gradient-based variants control only parts of this geometry.
% projections can reduce retain cross-leverage, while reweighting changes the effective residual without eliminating the underlying forget self-leverage.

% Intuitively, \(K_{ff}\) captures the uniqueness of forget samples under the retrained model geometry, while \(K_{rf}\) captures how strongly forget-directed changes transfer into retain outputs.

\begin{figure}[t]
\centering
\includegraphics[width=\linewidth]{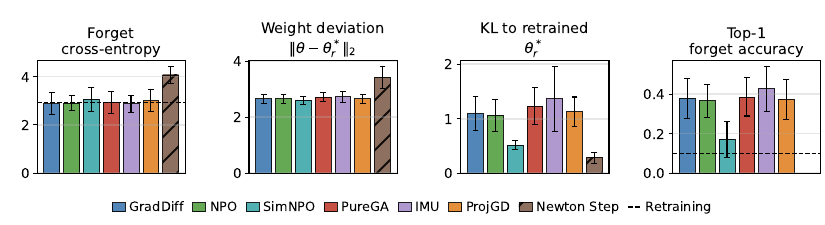}
\caption{\textbf{Loss and weights are not behavior.} MNIST\citep{lecun2010mnist} with multi-layer perceptron (MLP) model for class forgetting with $10$ seeds for various gradient-based methods and a Newton step unlearning. From left to right: forget cross-entropy (matched to the retrained reference, dashed), weight distance $\|\boldsymbol{\theta}-\thetaret\|_2$, KL$(p_{\thetaret}\|p_{\boldsymbol{\theta}})$ on $\Df$, and Top-1 forget accuracy. Setup is detailed in Appendix~\ref{app:exp_setup}.}
\label{fig:loss_not_behavior}
\end{figure}

Our key contributions are:
\begin{enumerate}[leftmargin=1.5em,itemsep=0.2em]
\item \textbf{Unified gap formula.} We derive closed-form output-space gaps for gradient-based unlearners, exact in linear regression and extended to nonlinear models under a local generalized Gauss-Newton (GGN) approximation \citep{schraudolph2002fast}. This gap shows a fundamental objective mismatch between common gradient-based methods and the gold standard.
\item \textbf{Two-leverage decomposition.} Our derivation reveals two important components of the gap: forget self-leverage, which captures the uniqueness of forget samples under the retrained model geometry, and retain cross-leverage, which captures how strongly forget-directed changes transfer into retain outputs.
\item \textbf{Forget-set diagnostics.} We validate each method's gap formula on MNIST and CIFAR-10, derive a closed-form sufficient rule for $\rho$, and empirically validate precomputable leverage diagnostics for LLM scale.
\end{enumerate}

\section{Related Work}
\subsection{Machine unlearning methods}
Early formulations of machine unlearning formalized data deletion through exact retraining, certified removal, or influence-style approximations to retraining \citep{cao2015towards,koh2017understanding,guo2019certified}. 
Newton-style methods are attractive because they target the retrained optimum directly, but they require Hessian or inverse-Hessian operations and rely on local approximations that can be costly or fragile at modern scale \citep{golatkar2020eternal,zhao2026win}. 
A related line linearizes the network around a reference point (pretrained or core-trained weights) and forgets exactly in the resulting linear, typically squared-loss, model. Examples include NTK-based scrubbing \citep{golatkar2020forgetting}, mixed-privacy forgetting \citep{golatkar2021mixed}, and parameter-efficient variants \citep{li2023fastntk}, made practical by fast empirical-NTK approximations \citep{wei2022more,mohamadi2023fast}. 
Within that linear model they are exact retraining, i.e.\ the \NS{} row of Table~\ref{tab:unified_gaps}, so their output gap to $\thetaret$ is governed by linearization error and the locality limits of Section~\ref{sec:nonlinear}. 
While Newton-style methods are limited in practice, the retrained-model Hessian they use as a preconditioner is the same geometric object that, as we show, implicitly governs the output-space gap of gradient-based unlearners.

Most practical large language model (LLM) unlearning methods instead optimize surrogate objectives. 
GA on the forget loss is the simplest example \citep{jang2023knowledge}; GradDiff adds retain descent to improve utility preservation \citep{liu2022continual}; 
NPO and SimNPO stabilize GradDiff-style unlearning through preference-style losses and reweighting \citep{zhang2024negative,fan2024simplicity}; 
IMU introduces influence-guided sample weights \citep{fan2025imu}; and GU-style projected methods reduce retain drift by projecting updates away from retain-gradient directions \citep{zhou2025geometric}. 
Another LLM baseline is Representation Misdirection Unlearning (RMU)\citep{li2024wmdp}, but it operates on intermediate representations rather than the output-loss gradient-based family analyzed in this work.

\subsection{Limits of loss-space and parameter-space evaluation}
\label{sec:rw_limits}
Several lines of existing work show from different angles that loss, output imitation on a limited evaluation set, and parameter agreement are not enough to certify unlearning. 
In general prediction problems, predictive multiplicity shows that competing models with similar loss can assign conflicting predictions \citep{marx2020predictive}. 
In unlearning, robustness studies show that superficial ``forgetting'' can be reversed by relearning or adversarial evaluation \citep{lee2025distillation,lucki2024adversarial,hu2024jogging,lynch2024eight,deeb2024unlearning}. 
Other work argues that parameter-space or algorithmic equivalence criteria are themselves delicate: verification-error analyses motivate proximity to retraining, while auditable definitions and retrain-equivalence impossibility results expose limits of parameter- or path-oblivious formulations \citep{thudi2022unrolling,thudi2022necessity,yu2025impossibility}. 
In the adjacent setting of catastrophic forgetting, \citet{shenfeld2025rlrazor} show that the forward KL to the original model on the new task predicts the loss of prior-task performance. We share the output-space viewpoint, but measure per-example gaps to the \emph{retrained} reference and decompose them by forget--retain geometry. 
% We adopt this literature as motivation rather than novelty: our goal is to characterize the structure of the output-space gap that remains.

\subsection{Retain-forget entanglement and geometry}
The closest work to ours studies why gradient-based unlearning fails when forget and retain data interact. 
\citet{mavrothalassitis2025ascent} attribute failures of ascent-based methods to statistical dependence between forget and retain sets. 
Other recent analyses study model resistance, stability, semantic entanglement, gradient conflict, and retain-gradient geometry \citep{chang2026some,cheng2026machine,zhou2025geometric,asif2025ofmu,xiao2026modeling}. 
These works diagnose retain-forget interaction through loss, gradient alignment, empirical instability, coherence, or semantic similarity. 
The main distinction of our work is the focus on output space: if an unlearned model is meant to approximate gold-standard retraining, then its outputs, not only its losses or parameters, are the relevant behavioral object.
Our unified analysis reveals the matrices through which this interaction propagates to outputs: forget self-leverage and retain cross-leverage under the retrained model geometry.

\section{Problem Setup}
\label{sec:setup}

In this section, we formally define the unlearning and the gold-standard retraining objectives.
We consider a training set $\D = \Dr \cup \Df$ with $|\D| = n$ and $|\Df| = m$, where $\Df$ is the \emph{forget set} whose influence is to be removed and $\Dr = \D \setminus \Df$ is the \emph{retain set}. Each sample $(\boldsymbol{x}_i, \boldsymbol{y}_i)$ pairs an input $\boldsymbol{x}_i$ with a label $\boldsymbol{y}_i$. A model $f(\boldsymbol{\theta}, \boldsymbol{x}) \in \mathbb{R}^c$ parameterized by $\boldsymbol{\theta} \in \mathbb{R}^d$ is trained on $\D$ under the $\ell_2$-regularized objective
\begin{equation}
    \mathcal{L}(\boldsymbol{\theta}) \;=\; \frac{1}{n}\sum_{i=1}^{n} \ell(\boldsymbol{\theta}, \boldsymbol{x}_i, \boldsymbol{y}_i) + \frac{\lambda}{2}\|\boldsymbol{\theta}\|^2,
    \qquad \lambda > 0,
    \label{eq:full_obj}
\end{equation}
where $\lambda$ is the regularization strength and $\ell(\cdot)$ is the loss function. The full-data optimum is $\thetafull = \arg\min_{\boldsymbol{\theta}} \mathcal{L}(\boldsymbol{\theta})$.

\paragraph{The gold standard.} The gold-standard unlearned model is the minimizer of an analogous objective restricted to $\Dr$. We adopt the full-data-normalized retain objective
\begin{equation}
    \Robj(\boldsymbol{\theta})
    \;:=\;
    \frac{1}{n}\sum_{i \in \Dr} \ell(\boldsymbol{\theta}, \boldsymbol{x}_i, \boldsymbol{y}_i)
    + \frac{\lambda}{2}\|\boldsymbol{\theta}\|^2,
    \qquad
    \thetaret \;:=\; \arg\min_{\boldsymbol{\theta}} \Robj(\boldsymbol{\theta}),
    \label{eq:retain_obj}
\end{equation}
which has the same minimizer as the standard $\frac{1}{n-m}$-normalized retain objective $\frac{1}{n-m}\sum_{i \in \Dr} \ell(\boldsymbol{\theta}, \boldsymbol{x}_i, \boldsymbol{y}_i) + \frac{\lambda_r}{2}\|\boldsymbol{\theta}\|^2$ with a scaled regularization strength $\lambda_r = \frac{n}{n-m}\lambda$. 
We refer to $\thetaret$ as the \emph{gold standard}: any unlearning method ideally returns a model whose induced behavior on arbitrary inputs matches that of $\thetaret$. Define also the full-data-normalized forget loss
\begin{equation}
    \Fobj(\boldsymbol{\theta}) \;:=\; \frac{1}{n}\sum_{j \in \Df} \ell(\boldsymbol{\theta}, \boldsymbol{x}_j, \boldsymbol{y}_j),
    \label{eq:forget_obj}
\end{equation}
so that $\mathcal{L} = \Robj + \Fobj$.

\paragraph{Gradient-objective unlearning.} Most existing optimization-based unlearners minimize a controlled retain--forget trade-off
\begin{equation}
    \boldsymbol{\theta}_{\rho}^{\mathrm{GDiff}}
    \;:=\; \arg\min_{\boldsymbol{\theta}}\;\Robj(\boldsymbol{\theta}) - \rho\,\Fobj(\boldsymbol{\theta}),
    \qquad \rho > 0,
    \label{eq:gdiff_obj}
\end{equation}
recovering retraining at $\rho = 0$ and adding active forget-ascent for $\rho > 0$. Variants such as NPO and SimNPO replace $\Fobj$ by a per-sample reweighted version, IMU incorporates per-sample influence weighting on the forget term, and Pure GA omits $\Robj$ entirely. 
Newton-style updates instead approximate $\thetaret$ in a single second-order step. Throughout, we treat $\thetaret$ as the reference solution and ask how each method's output behavior deviates from it. For nonconvex models the minimizers in~\eqref{eq:retain_obj} and~\eqref{eq:gdiff_obj} are not unique. We read them as local stationary points. In our MNIST and CIFAR-10 experiments $\thetaret$ is the warm-start retrain from $\thetafull$ (on TOFU, the released retain model), and in the fixed-point validation of Section~\ref{sec:exp_fixedpoint} each method's solution is the stationary point on the branch through $\thetaret$. For a parameter $\boldsymbol{\theta}$ and evaluation inputs $\boldsymbol{X}_a$, we call $f(\boldsymbol{\theta}, \boldsymbol{X}_a) - f(\thetaret, \boldsymbol{X}_a)$ the \emph{output-space gap} of $\boldsymbol{\theta}$ on $\boldsymbol{X}_a$, and summarize its size by per-coordinate RMS or output KL.

\section{Output-Space Gap Analysis}
\label{sec:analysis}
In this section, we derive closed-form expressions for the output-space gap between gradient-based unlearners and the gold-standard retrained model $\thetaret$, first under a linear squared-loss setting and then under a local nonlinear approximation.
Our following analysis reveals the key role of the retrained-model-geometry-based leverage matrices in structuring this gap, and it subsumes various gradient-based methods under a unified framework parametrized by geometry, reweighting, and projection.

\subsection{Linear squared-loss analysis}
\label{sec:linear}

We first analyze the output-space gap in a linear regression setting, which admits a closed-form solution and clean geometric interpretation that reveals two key factors influencing the gap.
To make the gap explicit, we begin with the squared-loss linear setting $f(\boldsymbol{\theta}, \boldsymbol{x}) = \boldsymbol{x}^\top \boldsymbol{\theta}$, for $\boldsymbol{x} \in \mathbb{R}^d$, $\boldsymbol{\theta} \in \mathbb{R}^d$, 
and $\ell(\boldsymbol{\theta}, \boldsymbol{x}, y) = \tfrac{1}{2}(y - \boldsymbol{x}^\top\boldsymbol{\theta})^2$. Let $\boldsymbol{X}_r \in \mathbb{R}^{(n-m)\times d}$ and $\boldsymbol{X}_f \in \mathbb{R}^{m\times d}$ be the retain and forget data matrices, respectively, with labels $\boldsymbol{y}_r \in \mathbb{R}^{n-m}, \boldsymbol{y}_f \in \mathbb{R}^{m}$.
The full-data objective is
\begin{equation}
    \mathcal{L}(\boldsymbol{\theta})
    =
    \frac{1}{2n}\|\boldsymbol{y}_r-\boldsymbol{X}_r\boldsymbol{\theta}\|_2^2
    + \frac{1}{2n}\|\boldsymbol{y}_f-\boldsymbol{X}_f\boldsymbol{\theta}\|_2^2
    + \frac{\lambda}{2}\|\boldsymbol{\theta}\|_2^2.
\end{equation}

Define the \emph{retain-Hessian} and the two natural \emph{forget residuals}, evaluated at the full-data optimum and the retrained optimum, respectively, as
\begin{equation}
    \boldsymbol{H}_r \;:=\; \tfrac{1}{n} \boldsymbol{X}_r^\top \boldsymbol{X}_r + \lambda \boldsymbol{I}_d,
    \qquad
    \boldsymbol{e}_f^+ \;:=\; \boldsymbol{X}_f \thetafull - \boldsymbol{y}_f,
    \qquad
    \boldsymbol{e}_f^- \;:=\; \boldsymbol{X}_f \thetaret - \boldsymbol{y}_f.
    \label{eq:H_r_def} 
\end{equation}
The residual $\boldsymbol{e}_f^+$ is precomputable from $\thetafull$ alone, while $\boldsymbol{e}_f^-$ is oracle-evaluable from the retrained model.
For any arbitrary evaluation data $\boldsymbol{X}_a$, we define the \emph{leverage matrix}
\begin{equation}
    \boldsymbol{K}_{af} \;:=\; \boldsymbol{X}_a \boldsymbol{H}_r^{-1} \boldsymbol{X}_f^\top,
    \label{eq:leverage}
\end{equation}
where as Theorem \ref{thm:retrain_linear_gap} shows, the leverage matrix $\boldsymbol{K}_{af}$ captures the influence of unlearning the forget samples on the output on arbitrary data $\boldsymbol{X}_a$ under the retrained model geometry.
Two special cases of the leverage matrix are (i) \emph{forget self-leverage} $\boldsymbol{K}_{ff} = \boldsymbol{X}_f \boldsymbol{H}_r^{-1} \boldsymbol{X}_f^\top$ : its diagonal entries capture the uniqueness of individual forget samples under this geometry, while its off-diagonal entries capture overlap among forget samples;
 and (ii) \emph{retain cross-leverage} $\boldsymbol{K}_{rf} = \boldsymbol{X}_r \boldsymbol{H}_r^{-1} \boldsymbol{X}_f^\top$\, which measures how forget-directed changes transfer into outputs on the retain set.

The two residuals $\boldsymbol{e}_f^+$ and $\boldsymbol{e}_f^-$ are not independent quantities; they are linked by an explicit invertible transformation through $\boldsymbol{K}_{ff}$, which we use throughout to translate between equivalent forms.

\begin{lemma}[Forget residual conversion]
\label{lem:residual_conversion}
Suppose $\boldsymbol{H}_r$ is positive definite. Then the two forget residuals defined in~\eqref{eq:H_r_def} are linked by an explicit invertible transformation:
\begin{equation}
    \boldsymbol{e}_f^- \;=\; \big(\boldsymbol{I}_m + \tfrac{1}{n}\boldsymbol{K}_{ff}\big)\,\boldsymbol{e}_f^+,
    \qquad
    \boldsymbol{e}_f^+ \;=\; \big(\boldsymbol{I}_m + \tfrac{1}{n}\boldsymbol{K}_{ff}\big)^{-1}\boldsymbol{e}_f^-.
    \label{eq:residual_conversion}
\end{equation}
In particular, $\boldsymbol{e}_f^+ = \boldsymbol{0} \iff \boldsymbol{e}_f^- = \boldsymbol{0}$.
\end{lemma}

The proof is in Appendix~\ref{app:proof_residual_conversion}; it is a one-line consequence of the normal-equation difference between $\thetafull$ and $\thetaret$, evaluated at $\boldsymbol{X}_f$.

\begin{theorem}[Retrained output deviation]
\label{thm:retrain_linear_gap}
Suppose $\boldsymbol{H}_r$ is positive definite. Then for any evaluation data $\boldsymbol{X}_a$, the output gap of the retrained model from the full-data optimum admits two equivalent closed forms:
\begin{align}
    \boldsymbol{X}_a\thetaret - \boldsymbol{X}_a\thetafull
    &\;=\; \tfrac{1}{n}\,\boldsymbol{K}_{af}\,\boldsymbol{e}_f^+
    \label{eq:retrain_linear_gap}\\
    &\;=\; \tfrac{1}{n}\,\boldsymbol{K}_{af}\big(\boldsymbol{I}_m + \tfrac{1}{n}\boldsymbol{K}_{ff}\big)^{-1}\boldsymbol{e}_f^-,
    \label{eq:retrain_linear_gap_ret}
\end{align}
where the second form follows from the first by Lemma~\ref{lem:residual_conversion}.
\end{theorem}

The proof is at Appendix~\ref{app:proof_retrain_linear_gap}. Equation~\eqref{eq:retrain_linear_gap} reveals that $\boldsymbol{K}_{af}$ captures how much of the forget residual $\boldsymbol{e}_f^+$ is transferred to the change in output on $\boldsymbol{X}_a$ when we retrain on $\Dr$, and the $\boldsymbol{e}_f^-$ form is its dual through $\boldsymbol{K}_{ff}$.

\begin{theorem}[Linear GradDiff output gap]
\label{thm:linear_gap}
Suppose $\boldsymbol{H}_r - \tfrac{\rho}{n}\boldsymbol{X}_f^\top\boldsymbol{X}_f$ is positive definite (equivalently $\tfrac{\rho}{n}\lambda_{\max}(\boldsymbol{K}_{ff}) < 1$, where $\lambda_{\max}(\cdot)$ captures the maximum eigenvalue, see Appendix~\ref{app:proof_stability}). Then for any evaluation data $\boldsymbol{X}_a$, the output gap of GradDiff from $\thetaret$ has the closed form
\begin{equation}
    \boldsymbol{X}_a\boldsymbol{\theta}_{\rho}^{\mathrm{GDiff}} - \boldsymbol{X}_a\thetaret
    \;=\; \frac{\rho}{n}\,\boldsymbol{K}_{af}\Big(\boldsymbol{I}_m - \tfrac{\rho}{n}\boldsymbol{K}_{ff}\Big)^{-1}\Big(\boldsymbol{I}_m + \tfrac{1}{n}\boldsymbol{K}_{ff}\Big)\boldsymbol{e}_f^+.
    \label{eq:linear_gap_full}
\end{equation}
Equivalently, in terms of the retrained-model forget residual $\boldsymbol{e}_f^- = \boldsymbol{X}_f\thetaret - \boldsymbol{y}_f$,
\begin{equation}
    \boldsymbol{X}_a\boldsymbol{\theta}_{\rho}^{\mathrm{GDiff}} - \boldsymbol{X}_a\thetaret
    \;=\; \frac{\rho}{n}\,\boldsymbol{K}_{af}\Big(\boldsymbol{I}_m - \tfrac{\rho}{n}\boldsymbol{K}_{ff}\Big)^{-1}\boldsymbol{e}_f^-.
    \label{eq:linear_gap}
\end{equation}
\end{theorem}

The proof is at Appendix~\ref{app:proof_linear_gap}.
The two forms ~\eqref{eq:linear_gap_full} and ~\eqref{eq:linear_gap} are equivalent by the residual conversion identity in Lemma~\ref{lem:residual_conversion}.
Eq. ~\eqref{eq:linear_gap} makes the structural decomposition cleanest: the retrained-model forget residual $\boldsymbol{e}_f^-$ is amplified by the inverse stability factor $(\boldsymbol{I}_m - \tfrac{\rho}{n}\boldsymbol{K}_{ff})^{-1}$, which captures how forget-directed perturbations are amplified under the retrained model geometry, and then transferred to the output gap on $\boldsymbol{X}_a$ by the leverage matrix $\boldsymbol{K}_{af}$.
Eq. ~\eqref{eq:linear_gap_full}, on the other hand, reveals that the gap is computable entirely from $\thetafull$ without any retraining since we can compute the retain Hessian $\boldsymbol{H}_r=\boldsymbol{H}-\boldsymbol{H}_f$, where $\boldsymbol{H}=\tfrac{1}{n}\boldsymbol{X}^\top\boldsymbol{X}+\lambda\boldsymbol{I}_d$ is the full Hessian and $\boldsymbol{H}_f=\tfrac{1}{n}\boldsymbol{X}_f^\top\boldsymbol{X}_f$ is the forget Hessian, making~\eqref{eq:linear_gap_full} an \emph{a priori} diagnostic that bounds the output-space gap on any candidate forget set $\Df$ before any unlearning is run.
We adopt the Eq. ~\eqref{eq:linear_gap} form in the following analysis for readability; all bounds and takeaways translate immediately to $\boldsymbol{e}_f^+$ via the same identity.

The two corollaries below summarize the practical norm bound and the separation between forget and retain output deviations.

\begin{corollary}[Stability of the effective residual]
\label{lem:stability}
Under the assumption of Theorem~\ref{thm:linear_gap},
\begin{equation}
    \big\| \boldsymbol{X}_a\boldsymbol{\theta}_{\rho}^{\mathrm{GDiff}} - \boldsymbol{X}_a\thetaret \big\|_2
    \;\le\;
    \frac{\rho}{n}\,\frac{\|\boldsymbol{K}_{af}\|_2\,\|\boldsymbol{e}_f^-\|_2}{1 - \tfrac{\rho}{n}\lambda_{\max}(\boldsymbol{K}_{ff})},
    \label{eq:linear_norm_bound}
\end{equation}
or equivalently, in terms of the precomputable full-data residual,
\begin{equation}
    \big\| \boldsymbol{X}_a\boldsymbol{\theta}_{\rho}^{\mathrm{GDiff}} - \boldsymbol{X}_a\thetaret \big\|_2
    \;\le\;
    \frac{\rho}{n}\,\frac{\|\boldsymbol{K}_{af}\|_2\,\big(1 + \tfrac{1}{n}\lambda_{\max}(\boldsymbol{K}_{ff})\big)\,\|\boldsymbol{e}_f^+\|_2}{1 - \tfrac{\rho}{n}\lambda_{\max}(\boldsymbol{K}_{ff})}.
    \label{eq:linear_norm_bound_full}
\end{equation}
\end{corollary}

\begin{corollary}[Forget and retain output deviations]
\label{lem:two_channels}
Under the assumption of Theorem~\ref{thm:linear_gap}, the GradDiff output gap on the forget and retain sets is
\begin{align}
    \boldsymbol{X}_f\boldsymbol{\theta}_{\rho}^{\mathrm{GDiff}} - \boldsymbol{X}_f\thetaret
    &= \tfrac{\rho}{n}\,\boldsymbol{K}_{ff}\big(\boldsymbol{I}_m - \tfrac{\rho}{n}\boldsymbol{K}_{ff}\big)^{-1}\boldsymbol{e}_f^-,
    \label{eq:gap_forget}\\
    \boldsymbol{X}_r\boldsymbol{\theta}_{\rho}^{\mathrm{GDiff}} - \boldsymbol{X}_r\thetaret
    &= \tfrac{\rho}{n}\,\boldsymbol{K}_{rf}\big(\boldsymbol{I}_m - \tfrac{\rho}{n}\boldsymbol{K}_{ff}\big)^{-1}\boldsymbol{e}_f^-.
    \label{eq:gap_retain}
\end{align}
The corresponding $\boldsymbol{e}_f^+$ forms follow from Lemma~\ref{lem:residual_conversion}.
\end{corollary}

Corollary~\ref{lem:two_channels} delivers the central design principle of this paper: \textbf{$\boldsymbol{K}_{ff}$ and $\boldsymbol{K}_{rf}$ capture qualitatively different geometric information about the forget set and act asymmetrically on the output gap.}
At a high level, $\boldsymbol{K}_{ff} = \boldsymbol{X}_f \boldsymbol{H}_r^{-1} \boldsymbol{X}_f^\top$ is an \emph{intrinsic} property of the forget set under the retrained-model geometry: its diagonal entries measure how distinctive each forget sample is once the retain data is accounted for, and its eigenvalues control how strongly forget-directed perturbations are amplified by $(\boldsymbol{I} - \tfrac{\rho}{n}\boldsymbol{K}_{ff})^{-1}$. By contrast, $\boldsymbol{K}_{rf} = \boldsymbol{X}_r \boldsymbol{H}_r^{-1} \boldsymbol{X}_f^\top$ is a \emph{cross-set} alignment quantity, capturing how forget-directed updates project onto retain directions; it carries no information about the forget set in isolation.
In the next section, we show that most existing gradient-based methods only partially account for the leverage geometry: they control the effective forget residual through reweighting or projection, but they do not fully eliminate the underlying forget self-leverage or retain cross-leverage.

\subsection{Unified method view}
\label{sec:unified}

So far our analysis has focused on the GradDiff method. In this section, we extend the analysis to a broader class of gradient-based unlearning methods. 
Each method differs only in (i) its \emph{effective leverage matrix} (retrained-model Hessian $\boldsymbol{H}_r^{-1}$ vs.\ Euclidean $\lambda^{-1}\boldsymbol{I}$) and (ii) a per-sample \emph{reweighting} matrix or fixed projector. Define the diagonal weighting matrices
\begin{equation}
    \boldsymbol{W}_{\mathrm{NPO}}^- \;=\; \mathrm{diag}\!\big(\sigma(z_j^-)\big)_{j=1}^m,
    \quad
    \boldsymbol{W}_{\mathrm{SimNPO}}^- \;=\; \mathrm{diag}\!\big(\sigma(\tilde{z}_j^-)\big)_{j=1}^m,
    \quad
    \boldsymbol{A}_{\mathrm{IMU}} \;=\; \mathrm{diag}(\alpha_j)_{j=1}^m,
\end{equation}
where $\sigma(t) = (1+e^{-t})^{-1}$ is the logistic sigmoid, $z_j^-$, $\tilde z_j^-$ are the NPO/SimNPO logit arguments evaluated at $\thetaret$, and $\alpha_j = \|\boldsymbol{H}_{\thetafull}^{-1}\nabla_{\boldsymbol{\theta}}\ell(\thetafull, \boldsymbol{x}_j, y_j)\|_2$ is IMU's per-sample influence norm, where $\boldsymbol{H}_{\thetafull}$ is the full-data Hessian at $\thetafull$.
Explicitly, for the squared-loss surrogate, $z_j^- = -\beta_{\mathrm{NPO}}\big(\ell_j(\thetaret)-\ell_j(\boldsymbol{\theta}_{\mathrm{ref}})\big)$ and $\tilde z_j^- = -\beta_{\mathrm{SimNPO}}\tilde\ell_j(\thetaret)$.
We adopt the convention that the (Sim)NPO objectives use the prefactor $\rho/\beta$ in front of the inner $\log\sigma(\beta\,\cdot)$ term (Appendix~\ref{app:proof_unified}), so that the inverse-temperature $\beta>0$ that arises from differentiating the inner sigmoid cancels with the prefactor and the effective coefficient appearing in Table~\ref{tab:unified_gaps} is exactly $\rho$ with the $\beta$-free weights above; under any other convention $\beta$ would simply be absorbed into $\rho$. Let also $\boldsymbol{G}_{ff} = \boldsymbol{X}_f\boldsymbol{X}_f^\top$ and $\boldsymbol{G}_{af} = \boldsymbol{X}_a\boldsymbol{X}_f^\top$ be Euclidean Gram matrices. 
For ProjGD, let $\boldsymbol{P}_r$ be a fixed projector that removes retain-gradient-aligned directions from the forget gradient, and define the projected leverage matrices as
\[
    \boldsymbol{K}^{\mathrm{P}}_{af}
    :=
    \boldsymbol{X}_a\boldsymbol{H}_r^{-1}\boldsymbol{P}_r\boldsymbol{X}_f^\top,
    \qquad
    \boldsymbol{K}^{\mathrm{P}}_{ff}
    :=
    \boldsymbol{X}_f\boldsymbol{H}_r^{-1}\boldsymbol{P}_r\boldsymbol{X}_f^\top.
\]

\begin{proposition}[Unified linear gaps]
\label{prop:unified_linear}
For the fixed-weight or frozen-local linear surrogate specified above, whenever the displayed inverse exists, each method's stationary output gap from $\thetaret$ admits the closed form in Table~\ref{tab:unified_gaps}. Invertibility ensures a unique stationary solution; positive-definite Hessians additionally ensure unique minima for the quadratic objectives.
\end{proposition}

\begin{table}[t]
\centering
\small
\setlength{\tabcolsep}{4pt}
\begin{tabular}{l l l}
\toprule
Method & Output gap $\;\boldsymbol{X}_a(\boldsymbol{\theta}_\rho^{\text{method}} - \thetaret)$ & Geometry \\
\midrule
GradDiff
& $\tfrac{\rho}{n}\,\boldsymbol{K}_{af}\,\big(\boldsymbol{I} - \tfrac{\rho}{n}\boldsymbol{K}_{ff}\big)^{-1}\boldsymbol{e}_f^-$
& retrained Hessian \\
NPO
& $\tfrac{\rho}{n}\,\boldsymbol{K}_{af}\,\big(\boldsymbol{I} - \tfrac{\rho}{n}\boldsymbol{W}^-_{\mathrm{NPO}}\boldsymbol{K}_{ff}\big)^{-1}\boldsymbol{W}^-_{\mathrm{NPO}}\boldsymbol{e}_f^-$
& retrained Hessian \\
SimNPO
& $\tfrac{\rho}{n}\,\boldsymbol{K}_{af}\,\big(\boldsymbol{I} - \tfrac{\rho}{n}\boldsymbol{W}^-_{\mathrm{SimNPO}}\boldsymbol{K}_{ff}\big)^{-1}\boldsymbol{W}^-_{\mathrm{SimNPO}}\boldsymbol{e}_f^-$
& retrained Hessian \\
Pure GA
& $-\,\boldsymbol{X}_a\thetaret - \tfrac{\rho}{n\lambda}\,\boldsymbol{G}_{af}\,\big(\boldsymbol{I} - \tfrac{\rho}{n\lambda}\boldsymbol{G}_{ff}\big)^{-1}\boldsymbol{y}_f$
& Euclidean Gram \\
IMU
& $-\,\boldsymbol{X}_a\thetaret - \tfrac{\rho}{n\lambda}\,\boldsymbol{G}_{af}\,\big(\boldsymbol{I} - \tfrac{\rho}{n\lambda}\boldsymbol{A}_{\mathrm{IMU}}\boldsymbol{G}_{ff}\big)^{-1}\boldsymbol{A}_{\mathrm{IMU}}\boldsymbol{y}_f$
& Euclidean Gram \\
ProjGD (fixed)
& $\tfrac{\rho}{n}\,\boldsymbol{K}^{\mathrm{P}}_{af}\,\big(\boldsymbol{I} - \tfrac{\rho}{n}\boldsymbol{K}^{\mathrm{P}}_{ff}\big)^{-1}\boldsymbol{e}_f^-$
& retrained Hessian (projected) \\
\NS{}
& $\boldsymbol{0}$
& retrained Hessian (Newton) \\
\bottomrule
\end{tabular}
\caption{Theoretical output gaps from $\thetaret$ in the linear setting. Each row assumes that its displayed inverse exists. The GradDiff, Pure GA, fixed-weight IMU, fixed-projector ProjGD, and \NS{} rows are exact for the displayed linear objectives; NPO and SimNPO are frozen-weight local forms, with weights evaluated at $\thetaret$. The geometry column distinguishes methods that route forget perturbations through $\boldsymbol{H}_r^{-1}$ (top group: GradDiff, NPO, SimNPO, fixed-projector ProjGD) from those that use raw Euclidean Gram structure (Pure GA, IMU). \NS{} recovers $\thetaret$ exactly in linear regression. Proofs in Appendix~\ref{app:proof_unified}.}
\label{tab:unified_gaps}
\end{table}

Table~\ref{tab:unified_gaps} shows that all methods share the same basic structure: an effective forget residual is formed by an inverse term on the forget set, then transferred to evaluation outputs through a method-specific geometry. 
GradDiff, NPO, SimNPO, and fixed-projector ProjGD use the retrained-Hessian geometry through $\boldsymbol{K}$-matrices, while Pure GA and IMU use Euclidean Gram matrices $\boldsymbol{G}$. 
Reweighting and projection change the inverse term and transfer map, but they do not remove the output-space gap in general. 
This distinction governs which leverage diagnostic should predict each method's behavior.

Concretely, for a fixed forget residual:
\begin{itemize}[leftmargin=1.5em,itemsep=0.2em]
\item \textbf{Forget self-leverage governs amplification on \emph{any} evaluation set.} In Eq.~\eqref{eq:linear_gap}, $\boldsymbol{K}_{ff}$ enters the inverse stability factor $(\boldsymbol{I}_m - \tfrac{\rho}{n}\boldsymbol{K}_{ff})^{-1}$. At fixed $\rho$, increasing $\|\boldsymbol{K}_{ff}\|_2$ increases its worst-case amplification within the stable region; realized gaps also depend on residual alignment and the evaluation map. On the forget set, $\boldsymbol{K}_{ff}$ additionally appears in the prefactor of~\eqref{eq:gap_forget}.
\item \textbf{Retain cross-leverage governs transfer to retain outputs.} For fixed $\boldsymbol{K}_{ff}$ and forget residual, the retain-gap bound scales linearly with $\|\boldsymbol{K}_{rf}\|_2$, while the realized gap depends on alignment with the effective residual. The forget gap does not involve $\boldsymbol{K}_{rf}$, and gaps on other evaluation sets depend on their respective $\boldsymbol{K}_{af}$.
\end{itemize}

\subsection{Local nonlinear approximation}
\label{sec:nonlinear}

In this section, we move away from the linear setting and show that the same structure persists for nonlinear models under a local approximation around $\thetaret$. 
Define for each forget sample $j$ the per-sample Jacobian, output-gradient, and output-space Hessian
\begin{equation}
    \boldsymbol{J}_j^- = \nabla_{\boldsymbol{\theta}} f(\thetaret, \boldsymbol{x}_j),
    \quad
    \boldsymbol{\delta}_j^- = \nabla_z \ell(z, y_j)\big|_{z = f(\thetaret, \boldsymbol{x}_j)},
    \quad
    \boldsymbol{B}_j^- = \nabla_z^2 \ell(z, y_j)\big|_{z = f(\thetaret, \boldsymbol{x}_j)},
\end{equation}
with stacked counterparts $\boldsymbol{J}_f^- \in \mathbb{R}^{mc \times d}$, $\boldsymbol{\delta}_f^- \in \mathbb{R}^{mc}$, $\boldsymbol{B}_f^- \in \mathbb{R}^{mc \times mc}$ (block-diagonal), where $c$ is the output dimension.
Let $\boldsymbol{H}_r^-$ be the positive-definite local retain curvature used by the approximation (the exact Hessian in an exact Taylor model, or the retain GGN curvature in the experiments), and define the nonlinear leverage matrices
\begin{equation}
    \boldsymbol{K}_{af}^{\mathrm{NL}} = \boldsymbol{J}_a^-\,(\boldsymbol{H}_r^-)^{-1}\,(\boldsymbol{J}_f^-)^\top,
    \qquad
    \boldsymbol{K}_{ff}^{\mathrm{NL}} = \boldsymbol{J}_f^-\,(\boldsymbol{H}_r^-)^{-1}\,(\boldsymbol{J}_f^-)^\top.
\end{equation}
Also define the forget GGN curvature and the local stability matrix
\begin{equation}
    \boldsymbol{H}_{f,\mathrm{G}}^-
    := \tfrac{1}{n}(\boldsymbol{J}_f^-)^\top \boldsymbol{B}_f^- \boldsymbol{J}_f^-,
    \qquad
    \boldsymbol{M}_\rho^-
    := \boldsymbol{H}_r^- - \rho\boldsymbol{H}_{f,\mathrm{G}}^- .
\end{equation}

\begin{proposition}[Local nonlinear GradDiff gap]
\label{prop:nonlinear_gap}
Suppose $\Robj$ and $\Fobj$ are twice continuously differentiable near $\thetaret$, and that $\boldsymbol{M}_\rho^-$ is positive definite. The local GGN/Taylor predictor for the GradDiff parameter gap is
\begin{equation}
    \widehat{\boldsymbol{\Delta}}_\rho
    := \rho(\boldsymbol{M}_\rho^-)^{-1}\nabla\Fobj(\thetaret)
    = \tfrac{\rho}{n}(\boldsymbol{H}_r^-)^{-1}(\boldsymbol{J}_f^-)^\top
      \Big(\boldsymbol{I}_{mc} - \tfrac{\rho}{n}\boldsymbol{B}_f^-\boldsymbol{K}_{ff}^{\mathrm{NL}}\Big)^{-1}\boldsymbol{\delta}_f^- .
    \label{eq:nonlinear_param_gap}
\end{equation}
Consequently, for any evaluation set $A$ with stacked Jacobian $\boldsymbol{J}_a^-$, the first-order predicted output gap is
\begin{equation}
    \widehat{\Delta f}_A
    := \boldsymbol{J}_a^-\widehat{\boldsymbol{\Delta}}_\rho
    = \frac{\rho}{n}\,\boldsymbol{K}_{af}^{\mathrm{NL}}\Big(\boldsymbol{I}_{mc} - \tfrac{\rho}{n}\boldsymbol{B}_f^-\boldsymbol{K}_{ff}^{\mathrm{NL}}\Big)^{-1}\boldsymbol{\delta}_f^-,
    \label{eq:nonlinear_gap}
\end{equation}
with the nonlinear evaluation remainder controlled by the local-error discussion below.
\end{proposition}

The proof (Appendix~\ref{app:proof_nonlinear}) is a Taylor expansion of the stationarity condition combined with the same push-through identity used in Theorem~\ref{thm:linear_gap}. 
Equation~\eqref{eq:nonlinear_gap} preserves every interpretive feature of the linear case: the residual term $\boldsymbol{\delta}_f^-$ measures forget mispredictions at $\thetaret$, the leverage $\boldsymbol{K}_{af}^{\mathrm{NL}}$ transfers them to evaluation outputs, and the effective-residual factor becomes ill-conditioned as the smallest singular value of $\boldsymbol{I}_{mc} - (\rho/n)\boldsymbol{B}_f^- \boldsymbol{K}_{ff}^{\mathrm{NL}}$ approaches zero. 
Analogous local approximations for the other methods in Table~\ref{tab:unified_gaps} are obtained by the same expansion and are stated in Appendix~\ref{app:nonlinear_unified}. 

\paragraph{Approximation error.}
Equation~\eqref{eq:nonlinear_gap} is a linear-response approximation to how the GradDiff stationary point moves under a forget-set reweighting. This is the same perturbative object studied by influence functions and the infinitesimal jackknife; finite-sample error bounds for such linearized reweighting approximations require local Hessian invertibility and Lipschitz/Taylor control of the gradients \citep{koh2017understanding,giordano2019swiss}. In our notation, the error has the same form: it is amplified by the local conditioning of $\boldsymbol{M}_\rho^-$, is second order in the displacement when exact Hessians are used, and gains an additional first-order term when the exact Hessian is replaced by a GGN curvature. A representative bound and derivation are given in Appendix~\ref{app:proof_nonlinear_error}; we use it only to justify the locality of the diagnostic, not as a new tight guarantee.

Our formulas give exact stationary equations in linear regression, but introduce local approximations for nonlinear models. 
Moreover, practical implementations use stochastic finite-step updates, retain subsets, approximate Hessians, and sometimes dynamically updated weights or projectors. 
The approximation-error discussion above shows how these local errors enter the output gap; Section~\ref{sec:experiments} shows that the leverage matrices remain informative in practice despite such implementation noise.

\section{Experiments}
\label{sec:experiments}
In this section, we demonstrate how one may apply the leverage geometry analysis to understand and predict the behavior of unlearning methods in practice.
We experiment on MNIST with a small MLP to (i) measure the leverage geometry $(\boldsymbol{K}_{ff}, \boldsymbol{K}_{rf}, \boldsymbol{e}_f^-)$ that Section~\ref{sec:linear} predicts to govern the output-space gap, in two qualitatively different forgetting regimes, and (ii) compare the five unlearners' output-space distance to the gold-standard retrained model in both regimes. 
The motivating evidence that loss and weight-space distance are not faithful evaluation surfaces is summarized in Figure~\ref{fig:loss_not_behavior}; this section quantifies the geometry that underlies that dissociation. A scalable $\thetafull$-only leverage score at LLM scale on TOFU is reported in Appendix~\ref{app:exp_tofu}, and additional validation of the exact linear formula, the pointwise nonlinear predictor, and the stability bound appears in Appendix~\ref{app:exp_extra}. 
Full reproducibility details are in Appendix~\ref{app:exp_setup}. 
Output deviations are reported as per-coordinate logit root-mean-square (RMS), $\sqrt{\mathbb{E}\|\Delta z\|_2^2/c}$, where $\Delta z$ is the output logits difference measured against the reference $\thetaret$.

\subsection{Class vs.\ random forgetting: leverage geometry and method behavior}
\label{sec:exp_regime}

We construct two regimes of forget set on MNIST. In the \emph{class} regime, $\Df$ contains all instances of one class in the $4{,}000$-sample training subset (so $|\Df|\approx 350$--$490$, reflecting natural class-size variation across the seed's subsample rather than multiple set sizes per class); 
in the \emph{random} regime, $\Df$ is a uniformly random training subsample of size $|\Df|=200$. For each forget set we form $\thetaret$ by fixed-step retraining and estimate the leverage geometry using a retain-normalized GGN surrogate (Appendix~\ref{app:exp_setup}). The four iterative unlearners take $500$ full-batch Adam steps; Newton Step takes one update. Aggregates are over $10$ classes $\times$ $3$ seeds (class regime) and $100$ random sets $\times$ $3$ seeds (random regime).

\input{tables/table_regime_compare.tex}

The bottom block of Table~\ref{tab:regime_compare} reports the central observable, RMS$_f$ and RMS$_r$ from $\thetaret$. Three takeaways follow from this comparison.

\textbf{(i) The forget channel is consistently larger than the retain channel.} 
For every method and every regime, RMS$_f \ge$ RMS$_r$, with ratios $1.0$--$1.55$. 
Equations~\eqref{eq:gap_forget}--\eqref{eq:gap_retain} distinguish these channels: $\boldsymbol{K}_{ff}$ enters the forget gap both as a prefactor and inside the inverse stability factor, while $\boldsymbol{K}_{rf}$ transfers the effective residual to retain outputs. The observed ordering also depends on residual alignment.

\textbf{(ii) Class forgetting concentrates the residual; the leverage matrices grow much less.} 
Class-regime forget CE is similar for the four iterative methods ($4.368$--$4.404$), all with zero forget accuracy, yet RMS$_f$ differs by about $58\times$ between SimNPO ($0.009$) and ProjGD ($0.509$). Newton has lower CE ($2.972$), so its larger gap illustrates accuracy mismatch only. We summarize these metrics here and in Appendix~\ref{app:exp_setup} rather than add columns to Table~\ref{tab:regime_compare}. Going from random to class forgetting, $\|\boldsymbol{e}_f^-\|$ grows by $4.59\times$, whereas the surrogate estimates of $\lambda_{\max}(\boldsymbol{K}_{ff})$ and $\|\boldsymbol{K}_{rf}\|_F$ grow by $1.57\times$ and $1.32\times$: the residual changes much more than the leverage norms.
The iterative methods (GradDiff, NPO, SimNPO) scale RMS$_f$ by $1.8$--$3.8\times$, consistent with the residual ratio modulated by geometry and method weights.
SimNPO separates from NPO mainly in the class regime. There the forget CE near $\thetaret$ is large ($\approx4.4$ vs.\ $0.33$ for random sets), so the default $\beta=2.5$ saturates its weights ($\bar w\approx10^{-4}$), switching the forget term off and leaving the iterate near $\thetaret$, whereas NPO at $\beta=0.1$ keeps nearly uniform $O(1)$ weights and so behaves like GradDiff with a rescaled $\rho$. Swapping the two $\beta$ values swaps the behavior (Table~\ref{tab:beta_swap}). 
ProjGD and the Newton step have larger output gaps; Appendix~\ref{app:regime_method_discussion} discusses their behavior.
Because Eq.~\eqref{eq:linear_norm_bound_full} bounds the gap from $\thetafull$ quantities alone, the leverage geometry can still serve as an a-priori predictive diagnostic when a retrained reference is unavailable.

\textbf{(iii) Geometry motivates a per-set $\rho$ rule.}
The stability condition of Theorem~\ref{thm:linear_gap} requires $s := (\rho/n)\lambda_{\max}(\boldsymbol{K}_{ff}) < 1$, motivating $\rho = s\,n/\lambda_{\max}(\boldsymbol{K}_{ff})$.
In Table~\ref{tab:regime_compare}, we use $s=0.3$ with the surrogate kernels described above. This is an empirical heuristic: the per-set $\rho$ varies by an order of magnitude across random forget sets, yet optimization is stable across all $330$ forget sets without per-set grid search.
The positive-definiteness guarantee applies to the consistently normalized, unprojected frozen GGN models in Appendix~\ref{app:stability_rule}; locality up to $s=2$ remains empirical. Saturated SimNPO weights motivate the weighted kernel in Section~\ref{sec:exp_fixedpoint}.

\paragraph{LLM-scale evidence.} 
On TOFU with Llama-3.2-1B-Instruct, $K$-channel scores computed once at $\thetafull$ in a LoRA chart rank, before any unlearning run, the retain examples that GradDiff, NPO, and SimNPO will disturb most (Spearman $0.26$--$0.49$ over three seeds, far above the permutation null) and, less sharply, the forget examples they will treat least like retraining ($0.14$--$0.29$, null for NPO on forget01; Appendix~\ref{app:exp_tofu}). They are a cheap per-example signal of unlearning difficulty and spillover, on par with per-example loss and gradient-norm baselines. With curvature restricted to a LoRA chart and early-stopped endpoints they do not yet predict the size of the gap. Whether richer curvature estimates or endpoints nearer stationarity improve these diagnostics remains open.

\subsection{Local gap validation and controls}
\label{sec:exp_fixedpoint}

\input{tables/table_headline_grid.tex}

We validate the local GradDiff, NPO, and SimNPO predictors on MNIST and a head over frozen CIFAR-10 features (Appendix~\ref{app:exp_setup}). Their stationary endpoints are reached by Newton continuation from $\thetaret$ with $s=0.1$, and the corresponding predictors are compared with actual logit gaps on $\Df$ and $256$ retain anchors (Table~\ref{tab:headline_main}). These endpoints lie within $5\times10^{-3}$ of $\thetaret$ in parameter norm. Predictions track both channels and datasets ($P^K_f=0.89$--$0.98$, $P^K_r=0.72$--$0.97$, relative $\ell_2$ error $0.25$--$0.84$) and beat the Euclidean control throughout. As a control, ProjGD uses projected-gradient dynamics with a recomputed projector and is compared with the unprojected predictor~\eqref{eq:nonlinear_gap}: their directions correlate ($P^K_f=0.79$--$0.97$), but the actual gap is about $7$--$15\times$ smaller. Pure GA and IMU provide controls outside the local regime: their endpoints have output KL $0.7$--$3$ nats from $\thetaret$, versus at most $10^{-5}$ for the retain-anchored methods, and their predictors are uninformative (Section~\ref{sec:limitations}).

The local validation concerns the stationary branch through the specified retrained reference; initialization at the original model can reach a different branch where the predictor fails, as demonstrated in Appendix~\ref{app:exp_setup}.

\section{Limitations and Discussion}
\label{sec:limitations}

Our analysis is exact for linear regression, but nonlinear approximations and practical implementations introduce additional error. 
Although Appendix~\ref{app:proof_nonlinear_error} analyzes this error and Section~\ref{sec:exp_fixedpoint} and Appendix~\ref{app:exp_extra} validate the approximation empirically, the leverage matrices should be read as output-space risk diagnostics rather than exact predictors, and they are local to the specified retrained reference and to the stationary branch through it (Section~\ref{sec:setup}, Appendix~\ref{app:exp_setup}). 
Relatedly, Pure GA and IMU have no term anchoring their iterates near $\thetafull$. Their rows in Table~\ref{tab:unified_gaps} remain exact but contain $-\boldsymbol{X}_a\thetaret$, so their gap is of the order of the reference outputs and the local predictor is uninformative at their converged endpoints under either geometry (Section~\ref{sec:exp_fixedpoint}), while their early-stopped operating points (Figure~\ref{fig:loss_not_behavior}) are not stationary points and lie outside our analysis. Computing the full gap is also infeasible for large models: Appendix~\ref{app:exp_tofu_kff_krf} shows that LoRA-space diagonal-GGN approximations of $\boldsymbol{K}_{ff}$ and $\boldsymbol{K}_{rf}$ are informative on TOFU, but their modest correlations leave better estimators as an open problem.

These limitations suggest future methods that account for both forget self-leverage and retain cross-leverage when choosing $\rho$, reweighting forget samples, or projecting updates, rather than optimizing losses alone. The precomputable $\boldsymbol{e}_f^+$ form of Theorem~\ref{thm:retrain_linear_gap} also suggests LLM-scale audits that combine released-model forget residuals with sketched or low-rank $\boldsymbol{K}_{ff}$ estimates.

The broader impact is primarily diagnostic: output-space auditing can improve privacy, data-removal, copyright, and safety evaluations, but approximate diagnostics should complement behavioral evaluation, privacy auditing, and transparent assumptions rather than serve as compliance certificates.

\section{Conclusion}
We studied gradient-based machine unlearning through the output-space gap to gold-standard retraining. 
In linear regression, this gap admits a closed form; under a local GGN approximation, the same structure extends to nonlinear models. 
The resulting framework unifies many common gradient-based unlearners and identifies two leverage matrices that govern how forget-set residuals propagate: forget self-leverage, which controls behavior on the forget set, and retain cross-leverage, which controls transfer into retain outputs. 
Together, these matrices provide a sharper language for comparing existing methods, auditing forget sets, and designing better unlearning algorithms.

\begin{ack}
This work received no third-party funding, and we declare no competing interests.
\end{ack}

\bibliographystyle{plainnat}
\bibliography{references}

\appendix

\section{Proofs and Detailed Derivations}
\label{app:proofs}

\subsection{Auxiliary push-through identity}
\label{app:push_through}

\begin{lemma}[Push-through identity]
\label{lem:push_through}
Let $\boldsymbol{A}$, $\boldsymbol{U}$, $\boldsymbol{C}$, and $\boldsymbol{V}$ be conformable matrices. If $\boldsymbol{A}$ and $\boldsymbol{I}-\boldsymbol{C}\boldsymbol{V}\boldsymbol{A}^{-1}\boldsymbol{U}$ are invertible, then $\boldsymbol{A}-\boldsymbol{U}\boldsymbol{C}\boldsymbol{V}$ is invertible and
\[
    (\boldsymbol{A}-\boldsymbol{U}\boldsymbol{C}\boldsymbol{V})^{-1}\boldsymbol{U}
    =
    \boldsymbol{A}^{-1}\boldsymbol{U}
    \big(\boldsymbol{I}-\boldsymbol{C}\boldsymbol{V}\boldsymbol{A}^{-1}\boldsymbol{U}\big)^{-1}.
\]
\end{lemma}

\begin{proof}
By Sylvester's determinant identity, $\boldsymbol{I}-\boldsymbol{A}^{-1}\boldsymbol{U}\boldsymbol{C}\boldsymbol{V}$ is invertible whenever $\boldsymbol{I}-\boldsymbol{C}\boldsymbol{V}\boldsymbol{A}^{-1}\boldsymbol{U}$ is invertible; hence $\boldsymbol{A}-\boldsymbol{U}\boldsymbol{C}\boldsymbol{V}=\boldsymbol{A}(\boldsymbol{I}-\boldsymbol{A}^{-1}\boldsymbol{U}\boldsymbol{C}\boldsymbol{V})$ is invertible.
Start from
\[
    (\boldsymbol{A}-\boldsymbol{U}\boldsymbol{C}\boldsymbol{V})\boldsymbol{A}^{-1}\boldsymbol{U}
    =
    \boldsymbol{U}-\boldsymbol{U}\boldsymbol{C}\boldsymbol{V}\boldsymbol{A}^{-1}\boldsymbol{U}
    =
    \boldsymbol{U}\big(\boldsymbol{I}-\boldsymbol{C}\boldsymbol{V}\boldsymbol{A}^{-1}\boldsymbol{U}\big).
\]
Multiplying on the left by $(\boldsymbol{A}-\boldsymbol{U}\boldsymbol{C}\boldsymbol{V})^{-1}$ and on the right by $\big(\boldsymbol{I}-\boldsymbol{C}\boldsymbol{V}\boldsymbol{A}^{-1}\boldsymbol{U}\big)^{-1}$ gives the result.
\end{proof}

\subsection{Proof of Lemma~\ref{lem:residual_conversion}}
\label{app:proof_residual_conversion}

In the linear squared-loss setting the full-data objective is
\[
    \mathcal{L}(\boldsymbol{\theta})
    =
    \frac{1}{2n}\|\boldsymbol{y}_r-\boldsymbol{X}_r\boldsymbol{\theta}\|_2^2
    + \frac{1}{2n}\|\boldsymbol{y}_f-\boldsymbol{X}_f\boldsymbol{\theta}\|_2^2
    + \frac{\lambda}{2}\|\boldsymbol{\theta}\|_2^2,
\]
whose stationarity condition $\nabla\mathcal{L}(\thetafull)=\boldsymbol{0}$ reads
\begin{equation}
    \Big(\boldsymbol{H}_r + \tfrac{1}{n}\boldsymbol{X}_f^\top\boldsymbol{X}_f\Big)\thetafull
    \;=\; \boldsymbol{b}_r + \tfrac{1}{n}\boldsymbol{X}_f^\top\boldsymbol{y}_f,
    \qquad
    \boldsymbol{b}_r := \tfrac{1}{n}\boldsymbol{X}_r^\top\boldsymbol{y}_r.
    \label{eq:full_normal_eq}
\end{equation}
The retrained optimum satisfies $\nabla\Robj(\thetaret)=\boldsymbol{0}$, i.e.\ $\boldsymbol{H}_r\thetaret=\boldsymbol{b}_r$. Rearranging~\eqref{eq:full_normal_eq},
\[
    \boldsymbol{H}_r\thetafull
    \;=\; \boldsymbol{b}_r - \tfrac{1}{n}\boldsymbol{X}_f^\top\big(\boldsymbol{X}_f\thetafull - \boldsymbol{y}_f\big)
    \;=\; \boldsymbol{b}_r - \tfrac{1}{n}\boldsymbol{X}_f^\top\boldsymbol{e}_f^+.
\]
Subtracting from $\boldsymbol{H}_r\thetaret = \boldsymbol{b}_r$ and using positive-definiteness of $\boldsymbol{H}_r$ gives the \emph{parameter shift}
\begin{equation}
    \boldsymbol{H}_r(\thetaret - \thetafull) \;=\; \tfrac{1}{n}\boldsymbol{X}_f^\top\boldsymbol{e}_f^+
    \quad\Longrightarrow\quad
    \thetaret - \thetafull \;=\; \tfrac{1}{n}\,\boldsymbol{H}_r^{-1}\boldsymbol{X}_f^\top\boldsymbol{e}_f^+.
    \label{eq:param_shift}
\end{equation}
Evaluating~\eqref{eq:param_shift} at $\boldsymbol{X}_f$,
\[
    \boldsymbol{e}_f^- - \boldsymbol{e}_f^+
    \;=\; \boldsymbol{X}_f(\thetaret - \thetafull)
    \;=\; \tfrac{1}{n}\boldsymbol{X}_f\boldsymbol{H}_r^{-1}\boldsymbol{X}_f^\top\boldsymbol{e}_f^+
    \;=\; \tfrac{1}{n}\boldsymbol{K}_{ff}\,\boldsymbol{e}_f^+.
\]
Rearranging gives $\boldsymbol{e}_f^- = \big(\boldsymbol{I}_m + \tfrac{1}{n}\boldsymbol{K}_{ff}\big)\boldsymbol{e}_f^+$. Since $\boldsymbol{K}_{ff}$ is positive semidefinite, $\boldsymbol{I}_m + \tfrac{1}{n}\boldsymbol{K}_{ff} \succeq \boldsymbol{I}_m$ is invertible, yielding the inverse direction $\boldsymbol{e}_f^+ = \big(\boldsymbol{I}_m + \tfrac{1}{n}\boldsymbol{K}_{ff}\big)^{-1}\boldsymbol{e}_f^-$ and the equivalence $\boldsymbol{e}_f^+ = \boldsymbol{0} \iff \boldsymbol{e}_f^- = \boldsymbol{0}$. \hfill$\square$

\subsection{Proof of Theorem~\ref{thm:retrain_linear_gap}}
\label{app:proof_retrain_linear_gap}

Left-multiplying the parameter shift~\eqref{eq:param_shift} (established in the proof of Lemma~\ref{lem:residual_conversion}) by $\boldsymbol{X}_a$ and using the definition $\boldsymbol{K}_{af}=\boldsymbol{X}_a\boldsymbol{H}_r^{-1}\boldsymbol{X}_f^\top$ yields~\eqref{eq:retrain_linear_gap}:
\[
    \boldsymbol{X}_a\thetaret - \boldsymbol{X}_a\thetafull
    \;=\; \tfrac{1}{n}\,\boldsymbol{K}_{af}\,\boldsymbol{e}_f^+.
\]
The equivalent retrained-residual form~\eqref{eq:retrain_linear_gap_ret} follows by substituting Lemma~\ref{lem:residual_conversion} into~\eqref{eq:retrain_linear_gap}. \hfill$\square$

\subsection{Proof of Theorem~\ref{thm:linear_gap}}
\label{app:proof_linear_gap}

In the linear squared-loss setting,
\[
    \Robj(\boldsymbol{\theta})
    =
    \frac{1}{2n}\|\boldsymbol{y}_r-\boldsymbol{X}_r\boldsymbol{\theta}\|_2^2
    +
    \frac{\lambda}{2}\|\boldsymbol{\theta}\|_2^2,
    \qquad
    \Fobj(\boldsymbol{\theta})
    =
    \frac{1}{2n}\|\boldsymbol{y}_f-\boldsymbol{X}_f\boldsymbol{\theta}\|_2^2.
\]
Hence
\[
    \nabla\Robj(\boldsymbol{\theta})
    =
    \Big(\tfrac{1}{n}\boldsymbol{X}_r^\top\boldsymbol{X}_r+\lambda\boldsymbol{I}\Big)\boldsymbol{\theta}
    -
    \tfrac{1}{n}\boldsymbol{X}_r^\top\boldsymbol{y}_r
    =
    \boldsymbol{H}_r\boldsymbol{\theta}-\boldsymbol{b}_r,
    \qquad
    \nabla\Fobj(\boldsymbol{\theta})
    =
    \tfrac{1}{n}\boldsymbol{X}_f^\top(\boldsymbol{X}_f\boldsymbol{\theta}-\boldsymbol{y}_f),
\]
where $\boldsymbol{b}_r := \tfrac{1}{n}\boldsymbol{X}_r^\top\boldsymbol{y}_r$. The first-order condition for $\boldsymbol{\theta}_\rho^{\mathrm{GDiff}}$ is
\[
    \nabla\Robj(\boldsymbol{\theta}_\rho^{\mathrm{GDiff}})
    -
    \rho\nabla\Fobj(\boldsymbol{\theta}_\rho^{\mathrm{GDiff}})
    =
    \boldsymbol{0}.
\]
Substituting the gradients and collecting the terms involving $\boldsymbol{\theta}_\rho^{\mathrm{GDiff}}$ gives
\begin{equation*}
    \big(\boldsymbol{H}_r - \tfrac{\rho}{n}\boldsymbol{X}_f^\top\boldsymbol{X}_f\big)\boldsymbol{\theta}_\rho^{\mathrm{GDiff}}
    \;=\; \boldsymbol{b}_r - \tfrac{\rho}{n}\boldsymbol{X}_f^\top\boldsymbol{y}_f,
\end{equation*}
or, writing
\[
    \boldsymbol{M}_\rho
    :=
    \boldsymbol{H}_r - \tfrac{\rho}{n}\boldsymbol{X}_f^\top\boldsymbol{X}_f,
\]
\[
    \boldsymbol{M}_\rho\boldsymbol{\theta}_\rho^{\mathrm{GDiff}}
    =
    \boldsymbol{b}_r - \tfrac{\rho}{n}\boldsymbol{X}_f^\top\boldsymbol{y}_f.
\]
The retrained optimum satisfies $\nabla\Robj(\thetaret)=0$, so $\boldsymbol{H}_r\thetaret=\boldsymbol{b}_r$. Therefore
\[
    \boldsymbol{M}_\rho\thetaret
    =
    \Big(\boldsymbol{H}_r-\tfrac{\rho}{n}\boldsymbol{X}_f^\top\boldsymbol{X}_f\Big)\thetaret
    =
    \boldsymbol{b}_r-\tfrac{\rho}{n}\boldsymbol{X}_f^\top\boldsymbol{X}_f\thetaret.
\]
Subtracting this identity from the equation for $\boldsymbol{\theta}_\rho^{\mathrm{GDiff}}$ yields
\begin{equation*}
    \boldsymbol{M}_\rho(\boldsymbol{\theta}_\rho^{\mathrm{GDiff}}-\thetaret)
    =
    \Big(\boldsymbol{b}_r-\tfrac{\rho}{n}\boldsymbol{X}_f^\top\boldsymbol{y}_f\Big)
    -
    \Big(\boldsymbol{b}_r-\tfrac{\rho}{n}\boldsymbol{X}_f^\top\boldsymbol{X}_f\thetaret\Big)
    =
    \tfrac{\rho}{n}\boldsymbol{X}_f^\top(\boldsymbol{X}_f\thetaret-\boldsymbol{y}_f).
\end{equation*}
Using $\boldsymbol{e}_f^-=\boldsymbol{X}_f\thetaret-\boldsymbol{y}_f$, this becomes
\[
    \boldsymbol{M}_\rho(\boldsymbol{\theta}_\rho^{\mathrm{GDiff}}-\thetaret)
    =
    \tfrac{\rho}{n}\boldsymbol{X}_f^\top\boldsymbol{e}_f^-.
\]
The assumption $\boldsymbol{M}_\rho\succ0$ ensures that we can invert this system, so
\[
    \boldsymbol{\theta}_\rho^{\mathrm{GDiff}}-\thetaret
    =
    \tfrac{\rho}{n}\boldsymbol{M}_\rho^{-1}\boldsymbol{X}_f^\top\boldsymbol{e}_f^-.
\]
Applying Lemma~\ref{lem:push_through} with $\boldsymbol{A} = \boldsymbol{H}_r$, $\boldsymbol{U} = \boldsymbol{X}_f^\top$, $\boldsymbol{C} = \tfrac{\rho}{n}\boldsymbol{I}_m$, and $\boldsymbol{V} = \boldsymbol{X}_f$,
\[
    \boldsymbol{M}_\rho^{-1}\boldsymbol{X}_f^\top
    =
    \boldsymbol{H}_r^{-1}\boldsymbol{X}_f^\top
    \Big(\boldsymbol{I}_m-\tfrac{\rho}{n}\boldsymbol{X}_f\boldsymbol{H}_r^{-1}\boldsymbol{X}_f^\top\Big)^{-1}
    =
    \boldsymbol{H}_r^{-1}\boldsymbol{X}_f^\top
    \Big(\boldsymbol{I}_m-\tfrac{\rho}{n}\boldsymbol{K}_{ff}\Big)^{-1}.
\]
Substituting this expression into the parameter gap gives
\begin{equation*}
    \boldsymbol{\theta}_\rho^{\mathrm{GDiff}} - \thetaret
    \;=\; \tfrac{\rho}{n}\,\boldsymbol{H}_r^{-1}\boldsymbol{X}_f^\top\big(\boldsymbol{I}_m - \tfrac{\rho}{n}\boldsymbol{K}_{ff}\big)^{-1}\boldsymbol{e}_f^-.
\end{equation*}
Finally, left-multiplying by $\boldsymbol{X}_a$ and using $\boldsymbol{K}_{af}=\boldsymbol{X}_a\boldsymbol{H}_r^{-1}\boldsymbol{X}_f^\top$ gives
\[
    \boldsymbol{X}_a(\boldsymbol{\theta}_\rho^{\mathrm{GDiff}}-\thetaret)
    =
    \tfrac{\rho}{n}\boldsymbol{K}_{af}
    \Big(\boldsymbol{I}_m-\tfrac{\rho}{n}\boldsymbol{K}_{ff}\Big)^{-1}
    \boldsymbol{e}_f^-,
\]
which is \eqref{eq:linear_gap}. The equivalent full-data-residual form~\eqref{eq:linear_gap_full} follows by substituting Lemma~\ref{lem:residual_conversion}:
\[
    \boldsymbol{X}_a(\boldsymbol{\theta}_\rho^{\mathrm{GDiff}}-\thetaret)
    =
    \tfrac{\rho}{n}\boldsymbol{K}_{af}
    \Big(\boldsymbol{I}_m-\tfrac{\rho}{n}\boldsymbol{K}_{ff}\Big)^{-1}
    \Big(\boldsymbol{I}_m+\tfrac{1}{n}\boldsymbol{K}_{ff}\Big)\boldsymbol{e}_f^+.
\]
\qed

\subsection{Proof of Corollary~\ref{lem:stability}}
\label{app:proof_stability}

The norm bound~\eqref{eq:linear_norm_bound} follows from \eqref{eq:linear_gap} and submultiplicativity of $\|\cdot\|_2$, using $\|(\boldsymbol{I} - \tfrac{\rho}{n}\boldsymbol{K}_{ff})^{-1}\|_2 = (1 - \tfrac{\rho}{n}\lambda_{\max}(\boldsymbol{K}_{ff}))^{-1}$ when $\boldsymbol{K}_{ff} \succeq 0$ and $\tfrac{\rho}{n}\lambda_{\max}(\boldsymbol{K}_{ff}) < 1$. The equivalent bound~\eqref{eq:linear_norm_bound_full} follows analogously from~\eqref{eq:linear_gap_full} together with $\|\boldsymbol{I}_m + \tfrac{1}{n}\boldsymbol{K}_{ff}\|_2 = 1 + \tfrac{1}{n}\lambda_{\max}(\boldsymbol{K}_{ff})$ since $\boldsymbol{K}_{ff} \succeq 0$. The stability equivalence is obtained by conjugating $\boldsymbol{H}_r - \tfrac{\rho}{n}\boldsymbol{X}_f^\top\boldsymbol{X}_f \succ 0$ by $\boldsymbol{H}_r^{-1/2}$ and noting that the nonzero eigenvalues of $\boldsymbol{H}_r^{-1/2}\boldsymbol{X}_f^\top\boldsymbol{X}_f\boldsymbol{H}_r^{-1/2}$ coincide with those of $\boldsymbol{K}_{ff}$. \qed

\subsection{Proof of Proposition~\ref{prop:unified_linear} (per-method)}
\label{app:proof_unified}

All five non-trivial rows are obtained by the same recipe used for GradDiff in Appendix~\ref{app:proof_linear_gap}: (i) write the method's regularized linear objective and its first-order stationarity condition, (ii) subtract the corresponding stationarity equation for $\thetaret$ to obtain a linear system in $\boldsymbol{\theta} - \thetaret$, and (iii) apply the push-through identity (Lemma~\ref{lem:push_through}) to express the result in the leverage-matrix form of Table~\ref{tab:unified_gaps}. We give each row in turn.

\paragraph{GradDiff (recap).} Objective $\Robj(\boldsymbol{\theta}) - \rho\Fobj(\boldsymbol{\theta})$, stationarity $(\boldsymbol{H}_r - \tfrac{\rho}{n}\boldsymbol{X}_f^\top\boldsymbol{X}_f)\boldsymbol{\theta} = \boldsymbol{b}_r - \tfrac{\rho}{n}\boldsymbol{X}_f^\top\boldsymbol{y}_f$ where $\boldsymbol{b}_r := \tfrac{1}{n}\boldsymbol{X}_r^\top\boldsymbol{y}_r$ and $\boldsymbol{H}_r\thetaret = \boldsymbol{b}_r$. Subtracting and applying Lemma~\ref{lem:push_through} with $\boldsymbol{A}=\boldsymbol{H}_r$, $\boldsymbol{U}=\boldsymbol{X}_f^\top$, $\boldsymbol{C}=\tfrac{\rho}{n}\boldsymbol{I}_m$, $\boldsymbol{V}=\boldsymbol{X}_f$ gives the GradDiff row (full derivation in Appendix~\ref{app:proof_linear_gap}).

\paragraph{NPO and SimNPO.} The (Sim)NPO objectives replace the unweighted forget term $-\rho\Fobj$ by a per-sample preference-style loss
\begin{align*}
    \mathrm{NPO}: &\;\;
    \Robj(\boldsymbol{\theta}) - \tfrac{\rho}{\beta}\cdot\tfrac{1}{n}\sum_{j\in\Df}\log\sigma\!\big(\beta\,(\ell_j(\boldsymbol{\theta}) - \ell_j(\boldsymbol{\theta}_{\mathrm{ref}}))\big),\\
    \mathrm{SimNPO}: &\;\;
    \Robj(\boldsymbol{\theta}) - \tfrac{\rho}{\beta}\cdot\tfrac{1}{n}\sum_{j\in\Df}\log\sigma\!\big(\beta\,\tilde\ell_j(\boldsymbol{\theta})\big),
\end{align*}
where $\ell_j(\boldsymbol{\theta}) = \tfrac{1}{2}(y_j-\boldsymbol{x}_j^\top\boldsymbol{\theta})^2$ is the per-sample squared loss and $\tilde\ell_j$ is its length-normalized variant; $\boldsymbol{\theta}_{\mathrm{ref}}$ is the NPO reference parameter. Differentiating,
\[
    \nabla_{\boldsymbol{\theta}}\big[-\tfrac{\rho}{\beta}\log\sigma(\beta\, t_j(\boldsymbol{\theta}))\big]
    = -\rho\,\sigma(-\beta t_j(\boldsymbol{\theta}))\,\nabla_{\boldsymbol{\theta}} t_j(\boldsymbol{\theta}),
\]
which produces a per-sample weight $w_j(\boldsymbol{\theta}) = \sigma(z_j(\boldsymbol{\theta}))$ on the standard squared-loss gradient (the inverse-temperature factor $\beta$ from differentiating the inner $\log\sigma(\beta\,\cdot)$ exactly cancels with the $\rho/\beta$ prefactor of the objective, so the effective forget coefficient in the linear stationarity equation below is $\rho$ and the diagonal weights are $\beta$-free, matching Section~\ref{sec:unified}; under any alternative prefactor convention $\beta$ would simply be absorbed into $\rho$). Freezing the weights at $\thetaret$ yields the linear stationarity condition
\[
    (\boldsymbol{H}_r - \tfrac{\rho}{n}\boldsymbol{X}_f^\top\boldsymbol{W}^-\boldsymbol{X}_f)\boldsymbol{\theta} = \boldsymbol{b}_r - \tfrac{\rho}{n}\boldsymbol{X}_f^\top\boldsymbol{W}^-\boldsymbol{y}_f,
\]
with $\boldsymbol{W}^- = \boldsymbol{W}^-_{\mathrm{NPO}}$ or $\boldsymbol{W}^-_{\mathrm{SimNPO}}$. Subtracting the $\thetaret$ equation $\boldsymbol{H}_r\thetaret = \boldsymbol{b}_r$ gives $(\boldsymbol{H}_r - \tfrac{\rho}{n}\boldsymbol{X}_f^\top\boldsymbol{W}^-\boldsymbol{X}_f)(\boldsymbol{\theta}-\thetaret) = \tfrac{\rho}{n}\boldsymbol{X}_f^\top\boldsymbol{W}^-\boldsymbol{e}_f^-$. Applying Lemma~\ref{lem:push_through} with $\boldsymbol{C} = \tfrac{\rho}{n}\boldsymbol{W}^-$ and left-multiplying by $\boldsymbol{X}_a$ yields the (Sim)NPO row.

\paragraph{Pure GA.} The objective drops $\Robj$ and keeps only ridge regularization and forget ascent: $-\rho\Fobj(\boldsymbol{\theta}) + \tfrac{\lambda}{2}\|\boldsymbol{\theta}\|^2$. The first-order condition is $(\lambda\boldsymbol{I}_d - \tfrac{\rho}{n}\boldsymbol{X}_f^\top\boldsymbol{X}_f)\boldsymbol{\theta} = -\tfrac{\rho}{n}\boldsymbol{X}_f^\top\boldsymbol{y}_f$. Lemma~\ref{lem:push_through} with $\boldsymbol{A}=\lambda\boldsymbol{I}_d$, $\boldsymbol{U}=\boldsymbol{X}_f^\top$, $\boldsymbol{C}=\tfrac{\rho}{n}\boldsymbol{I}_m$, $\boldsymbol{V}=\boldsymbol{X}_f$ gives
$\boldsymbol{\theta}_\rho^{\mathrm{GA}} = -\tfrac{\rho}{n\lambda}\boldsymbol{X}_f^\top(\boldsymbol{I}_m - \tfrac{\rho}{n\lambda}\boldsymbol{G}_{ff})^{-1}\boldsymbol{y}_f$. Subtracting $\thetaret$ and left-multiplying by $\boldsymbol{X}_a$ yields the Pure-GA row.

\paragraph{IMU.} IMU rescales each forget term by a per-sample influence weight $\alpha_j$, giving the objective $-\rho\cdot\tfrac{1}{n}\sum_{j\in\Df}\alpha_j\,\ell_j(\boldsymbol{\theta}) + \tfrac{\lambda}{2}\|\boldsymbol{\theta}\|^2$ (without $\Robj$, matching the Pure-GA setting). The stationarity condition becomes the Pure-GA system with the influence diagonal $\boldsymbol{A}_{\mathrm{IMU}}$ inserted: $(\lambda\boldsymbol{I}_d - \tfrac{\rho}{n}\boldsymbol{X}_f^\top\boldsymbol{A}_{\mathrm{IMU}}\boldsymbol{X}_f)\boldsymbol{\theta} = -\tfrac{\rho}{n}\boldsymbol{X}_f^\top\boldsymbol{A}_{\mathrm{IMU}}\boldsymbol{y}_f$. The same push-through application with $\boldsymbol{C} = \tfrac{\rho}{n}\boldsymbol{A}_{\mathrm{IMU}}$ yields the IMU row. This is a fixed-weight local approximation; in the experiments the diagonal is instantiated once from $\thetafull$ using the final-layer proxy in Appendix~\ref{app:exp_setup}.

\paragraph{ProjGD.} The fixed-projector method applies $\boldsymbol{P}_r$ to the forget gradient at every step, so the stationarity condition is $\nabla\Robj(\boldsymbol{\theta}) - \rho\boldsymbol{P}_r\nabla\Fobj(\boldsymbol{\theta}) = \boldsymbol{0}$. In linear regression this gives $(\boldsymbol{H}_r - \tfrac{\rho}{n}\boldsymbol{P}_r\boldsymbol{X}_f^\top\boldsymbol{X}_f)\boldsymbol{\theta} = \boldsymbol{b}_r - \tfrac{\rho}{n}\boldsymbol{P}_r\boldsymbol{X}_f^\top\boldsymbol{y}_f$. Subtracting the $\thetaret$ equation and applying Lemma~\ref{lem:push_through} with $\boldsymbol{A}=\boldsymbol{H}_r$, $\boldsymbol{U}=\boldsymbol{P}_r\boldsymbol{X}_f^\top$, $\boldsymbol{C}=\tfrac{\rho}{n}\boldsymbol{I}_m$, $\boldsymbol{V}=\boldsymbol{X}_f$ yields the ProjGD row, with $\boldsymbol{K}^{\mathrm{P}}_{af}$ and $\boldsymbol{K}^{\mathrm{P}}_{ff}$ in place of $\boldsymbol{K}_{af}$ and $\boldsymbol{K}_{ff}$.

\paragraph{\NS{}.} Newton step solves $\thetafull + \boldsymbol{H}_r^{-1}\nabla\Fobj(\thetafull) = \thetaret$ exactly in linear regression with squared loss, using the same push-through identity as Theorem~\ref{thm:retrain_linear_gap} \citep{zhao2026win}. Hence the gap is zero. \qed

\subsection{Proof of Proposition~\ref{prop:nonlinear_gap}}
\label{app:proof_nonlinear}

Let $\boldsymbol{\Delta}_\rho = \boldsymbol{\theta}_\rho^{\mathrm{GDiff}} - \thetaret$. The first-order condition for $\boldsymbol{\theta}_\rho^{\mathrm{GDiff}}$ is $\nabla\Robj(\boldsymbol{\theta}_\rho^{\mathrm{GDiff}}) - \rho\nabla\Fobj(\boldsymbol{\theta}_\rho^{\mathrm{GDiff}}) = \boldsymbol{0}$. Using $\nabla\Robj(\thetaret) = \boldsymbol{0}$ and replacing the local gradients by their first-order GGN/Taylor model gives
\begin{equation*}
    \boldsymbol{H}_r^-\,\boldsymbol{\Delta}_\rho - \rho\,\boldsymbol{g}_f^- - \rho\,\boldsymbol{H}_{f,\mathrm{G}}^-\,\boldsymbol{\Delta}_\rho \;\approx\; \boldsymbol{0},
\end{equation*}
where $\boldsymbol{g}_f^- = \nabla\Fobj(\thetaret) = \tfrac{1}{n}(\boldsymbol{J}_f^-)^\top\boldsymbol{\delta}_f^-$ and $\boldsymbol{H}_{f,\mathrm{G}}^- = \tfrac{1}{n}(\boldsymbol{J}_f^-)^\top\boldsymbol{B}_f^-\boldsymbol{J}_f^-$. Solving this local linear system gives
\[
    \widehat{\boldsymbol{\Delta}}_\rho
    = \rho(\boldsymbol{H}_r^- - \rho\boldsymbol{H}_{f,\mathrm{G}}^-)^{-1}\boldsymbol{g}_f^-.
\]
Applying the same push-through identity with $\boldsymbol{A} = \boldsymbol{H}_r^-$, $\boldsymbol{U} = (\boldsymbol{J}_f^-)^\top$, $\boldsymbol{C} = \tfrac{\rho}{n}\boldsymbol{B}_f^-$, and $\boldsymbol{V} = \boldsymbol{J}_f^-$ yields
\begin{equation*}
    \widehat{\boldsymbol{\Delta}}_\rho
    = \tfrac{\rho}{n}\,(\boldsymbol{H}_r^-)^{-1}(\boldsymbol{J}_f^-)^\top\big(\boldsymbol{I}_{mc} - \tfrac{\rho}{n}\boldsymbol{B}_f^-\boldsymbol{K}_{ff}^{\mathrm{NL}}\big)^{-1}\boldsymbol{\delta}_f^-.
\end{equation*}
Left-multiplying by $\boldsymbol{J}_a^-$ gives the displayed output predictor in \eqref{eq:nonlinear_gap}. \qed

\subsection{A local nonlinear approximation-error bound}
\label{app:proof_nonlinear_error}

This appendix records a representative local bound for the approximation-error discussion in Section~\ref{sec:nonlinear}. Let $\boldsymbol{\Delta}_\rho = \boldsymbol{\theta}_\rho^{\mathrm{GDiff}} - \thetaret$ be any stationary GradDiff displacement in a neighborhood of $\thetaret$, and assume $\| (\boldsymbol{M}_\rho^-)^{-1} \|_2 \le \gamma$. Suppose that for all displacements $\boldsymbol{\Delta}$ in this neighborhood,
\begin{align*}
\|\nabla\Robj(\thetaret+\boldsymbol{\Delta}) - \boldsymbol{H}_r^-\boldsymbol{\Delta}\|_2
&\le \varepsilon_r\|\boldsymbol{\Delta}\|_2 + \tfrac{L_r}{2}\|\boldsymbol{\Delta}\|_2^2, \\
\|\nabla\Fobj(\thetaret+\boldsymbol{\Delta}) - \nabla\Fobj(\thetaret) - \boldsymbol{H}_{f,\mathrm{G}}^-\boldsymbol{\Delta}\|_2
&\le \varepsilon_f\|\boldsymbol{\Delta}\|_2 + \tfrac{L_f}{2}\|\boldsymbol{\Delta}\|_2^2,
\end{align*}
and that the evaluation map satisfies
\[
\|f(\thetaret+\boldsymbol{\Delta}, A)-f(\thetaret,A)-\boldsymbol{J}_a^-\boldsymbol{\Delta}\|_2
\le \tfrac{L_a}{2}\|\boldsymbol{\Delta}\|_2^2.
\]
Then the output-space predictor $\widehat{\Delta f}_A := \boldsymbol{J}_a^-\widehat{\boldsymbol{\Delta}}_\rho$ obeys
\begin{equation}
\big\| f(\boldsymbol{\theta}_\rho^{\mathrm{GDiff}},A)-f(\thetaret,A)-\widehat{\Delta f}_A \big\|_2
\le
\|\boldsymbol{J}_a^-\|_2\gamma
\Big((\varepsilon_r+\rho\varepsilon_f)\|\boldsymbol{\Delta}_\rho\|_2
+ \tfrac{L_r+\rho L_f}{2}\|\boldsymbol{\Delta}_\rho\|_2^2\Big)
+ \tfrac{L_a}{2}\|\boldsymbol{\Delta}_\rho\|_2^2 .
\label{eq:nonlinear_error_bound}
\end{equation}
For exact Hessians, $\varepsilon_r=\varepsilon_f=0$ under ordinary Lipschitz-Hessian smoothness, by which we mean that $\nabla^2\Robj$ and $\nabla^2\Fobj$ are $L_r$- and $L_f$-Lipschitz on the neighborhood, so a standard Taylor argument removes the linear remainder and leaves only the $L_r,L_f$ second-order terms. With GGN curvatures, $\varepsilon_r$ and $\varepsilon_f$ record the omitted model-Hessian terms.

Let $\boldsymbol{g}_f^- = \nabla\Fobj(\thetaret)$ and define the local-model remainders
\begin{align*}
\boldsymbol{r}_r(\boldsymbol{\Delta})
&:= \nabla\Robj(\thetaret+\boldsymbol{\Delta}) - \boldsymbol{H}_r^-\boldsymbol{\Delta}, \\
\boldsymbol{r}_f(\boldsymbol{\Delta})
&:= \nabla\Fobj(\thetaret+\boldsymbol{\Delta}) - \boldsymbol{g}_f^- - \boldsymbol{H}_{f,\mathrm{G}}^-\boldsymbol{\Delta}.
\end{align*}
Since $\boldsymbol{\Delta}_\rho$ is stationary for GradDiff,
\begin{align*}
\boldsymbol{0}
&= \nabla\Robj(\thetaret+\boldsymbol{\Delta}_\rho)
   - \rho\nabla\Fobj(\thetaret+\boldsymbol{\Delta}_\rho) \\
&= \boldsymbol{M}_\rho^-\boldsymbol{\Delta}_\rho
   - \rho\boldsymbol{g}_f^- + \boldsymbol{r}_r(\boldsymbol{\Delta}_\rho)
   - \rho\boldsymbol{r}_f(\boldsymbol{\Delta}_\rho).
\end{align*}
The predictor satisfies $\boldsymbol{M}_\rho^-\widehat{\boldsymbol{\Delta}}_\rho = \rho\boldsymbol{g}_f^-$. Subtracting the two equations gives
\[
\boldsymbol{\Delta}_\rho - \widehat{\boldsymbol{\Delta}}_\rho
= (\boldsymbol{M}_\rho^-)^{-1}\big(-\boldsymbol{r}_r(\boldsymbol{\Delta}_\rho)+\rho\boldsymbol{r}_f(\boldsymbol{\Delta}_\rho)\big).
\]
Using the assumed remainder bounds and $\|(\boldsymbol{M}_\rho^-)^{-1}\|_2\le\gamma$,
\[
\|\boldsymbol{\Delta}_\rho - \widehat{\boldsymbol{\Delta}}_\rho\|_2
\le
\gamma\Big((\varepsilon_r+\rho\varepsilon_f)\|\boldsymbol{\Delta}_\rho\|_2
+ \tfrac{L_r+\rho L_f}{2}\|\boldsymbol{\Delta}_\rho\|_2^2\Big).
\]
Finally,
\begin{align*}
&\big\| f(\boldsymbol{\theta}_\rho^{\mathrm{GDiff}},A)-f(\thetaret,A)-\boldsymbol{J}_a^-\widehat{\boldsymbol{\Delta}}_\rho \big\|_2 \\
&\qquad\le
\|\boldsymbol{J}_a^-\|_2\|\boldsymbol{\Delta}_\rho-\widehat{\boldsymbol{\Delta}}_\rho\|_2
+ \big\|f(\thetaret+\boldsymbol{\Delta}_\rho,A)-f(\thetaret,A)-\boldsymbol{J}_a^-\boldsymbol{\Delta}_\rho\big\|_2,
\end{align*}
and the evaluation-map remainder bound gives \eqref{eq:nonlinear_error_bound}. \qed

\subsection{Local nonlinear forms for the other methods}
\label{app:nonlinear_unified}

The local nonlinear rows are obtained from Table~\ref{tab:unified_gaps} by replacing linear features with Jacobians and inserting the output-Hessian/GGN factor $\boldsymbol{B}_f^-$. For NPO and SimNPO, weights are frozen at $\thetaret$ and their derivatives are omitted. Scalar per-sample weights are lifted to output blocks as $\widetilde{\boldsymbol{W}}=\boldsymbol{W}\otimes\boldsymbol{I}_c$ and $\widetilde{\boldsymbol{A}}=\boldsymbol{A}\otimes\boldsymbol{I}_c$. For fixed-projector ProjGD and the Euclidean methods, define
\[
    \boldsymbol{K}_{af}^{\mathrm{NL,P}}
    :=
    \boldsymbol{J}_a^-(\boldsymbol{H}_r^-)^{-1}\boldsymbol{P}_r(\boldsymbol{J}_f^-)^\top,
    \qquad
    \boldsymbol{K}_{ff}^{\mathrm{NL,P}}
    :=
    \boldsymbol{J}_f^-(\boldsymbol{H}_r^-)^{-1}\boldsymbol{P}_r(\boldsymbol{J}_f^-)^\top,
\]
\[
    \boldsymbol{G}_{af}^{\mathrm{NL}} := \boldsymbol{J}_a^-(\boldsymbol{J}_f^-)^\top,
    \qquad
    \boldsymbol{G}_{ff}^{\mathrm{NL}} := \boldsymbol{J}_f^-(\boldsymbol{J}_f^-)^\top.
\]
For the Euclidean rows, let
\[
    \boldsymbol{\eta}_f^- := \boldsymbol{\delta}_f^- - \boldsymbol{B}_f^-\boldsymbol{J}_f^-\thetaret.
\]
The shifted residual $\boldsymbol{\eta}_f^-$ appears in the Pure GA and IMU rows because their local linearization replaces the regression target $\boldsymbol{y}_f$ of Table~\ref{tab:unified_gaps} by its linearized analog $\boldsymbol{\delta}_f^- - \boldsymbol{B}_f^-\boldsymbol{J}_f^-\thetaret$; this is the per-sample object that plays the role of $-\boldsymbol{y}_f$ when the loss is taken to second order in the model output rather than in the parameter directly.
Whenever the displayed inverse exists, the leading-order local output gaps are summarized in Table~\ref{tab:nonlinear_unified}. All quantities are evaluated at $\thetaret$. The Pure GA and IMU rows are written in Euclidean-Gram form; because they do not use the retrained-Hessian $K$ geometry, the $K$-channel pointwise diagnostic in Appendix~\ref{app:exp_extra} is not the predictor for those rows.

\begin{table}[ht]
\centering
\scriptsize
\setlength{\tabcolsep}{3pt}
\resizebox{\textwidth}{!}{%
\begin{tabular}{l l l}
\toprule
Method & Local output gap $f(\boldsymbol{\theta}_\rho^{\text{method}}, A)-f(\thetaret,A)$ & Geometry \\
\midrule
GradDiff
& $\tfrac{\rho}{n}\boldsymbol{K}_{af}^{\mathrm{NL}}\big(\boldsymbol{I}-\tfrac{\rho}{n}\boldsymbol{B}_f^-\boldsymbol{K}_{ff}^{\mathrm{NL}}\big)^{-1}\boldsymbol{\delta}_f^-$
& retrained Hessian \\
NPO
& $\tfrac{\rho}{n}\boldsymbol{K}_{af}^{\mathrm{NL}}\big(\boldsymbol{I}-\tfrac{\rho}{n}\widetilde{\boldsymbol{W}}_{\mathrm{NPO}}^-\boldsymbol{B}_f^-\boldsymbol{K}_{ff}^{\mathrm{NL}}\big)^{-1}\widetilde{\boldsymbol{W}}_{\mathrm{NPO}}^-\boldsymbol{\delta}_f^-$
& retrained Hessian \\
SimNPO
& $\tfrac{\rho}{n}\boldsymbol{K}_{af}^{\mathrm{NL}}\big(\boldsymbol{I}-\tfrac{\rho}{n}\widetilde{\boldsymbol{W}}_{\mathrm{SimNPO}}^-\boldsymbol{B}_f^-\boldsymbol{K}_{ff}^{\mathrm{NL}}\big)^{-1}\widetilde{\boldsymbol{W}}_{\mathrm{SimNPO}}^-\boldsymbol{\delta}_f^-$
& retrained Hessian \\
Pure GA
& $-\boldsymbol{J}_a^-\thetaret + \tfrac{\rho}{n\lambda}\boldsymbol{G}_{af}^{\mathrm{NL}}\big(\boldsymbol{I}-\tfrac{\rho}{n\lambda}\boldsymbol{B}_f^-\boldsymbol{G}_{ff}^{\mathrm{NL}}\big)^{-1}\boldsymbol{\eta}_f^-$
& Euclidean Gram \\
IMU
& $-\boldsymbol{J}_a^-\thetaret + \tfrac{\rho}{n\lambda}\boldsymbol{G}_{af}^{\mathrm{NL}}\big(\boldsymbol{I}-\tfrac{\rho}{n\lambda}\widetilde{\boldsymbol{A}}_{\mathrm{IMU}}\boldsymbol{B}_f^-\boldsymbol{G}_{ff}^{\mathrm{NL}}\big)^{-1}\widetilde{\boldsymbol{A}}_{\mathrm{IMU}}\boldsymbol{\eta}_f^-$
& Euclidean Gram \\
ProjGD (fixed)
& $\tfrac{\rho}{n}\boldsymbol{K}_{af}^{\mathrm{NL,P}}\big(\boldsymbol{I}-\tfrac{\rho}{n}\boldsymbol{B}_f^-\boldsymbol{K}_{ff}^{\mathrm{NL,P}}\big)^{-1}\boldsymbol{\delta}_f^-$
& retrained Hessian (projected) \\
\NS{}
& $\boldsymbol{0}$ in the local quadratic retraining model
& retrained Hessian (Newton) \\
\bottomrule
\end{tabular}%
}
\caption{Local nonlinear counterparts of Table~\ref{tab:unified_gaps}; notation as defined in Section~\ref{sec:nonlinear}. NPO/SimNPO use weights frozen at $\thetaret$, omitting their derivatives. The \NS{} row omits higher-order and GGN-mismatch residuals.}
\label{tab:nonlinear_unified}
\end{table}
\section{Experimental Setup Details}
\label{app:exp_setup}

This appendix gives the full reproducibility detail for the experiments of Section~\ref{sec:experiments}: data, model, hardware, software, hyperparameters, and per-experiment protocols.

\paragraph{Hardware and software.} 
All MNIST/MLP experiments run in float64 on a single NVIDIA H100 NVL 96GB GPU within a Linux + CUDA $\geq$ 11.8 environment (PyTorch~2). The figures below were measured by re-running one seed of every experiment in its final configuration on an otherwise idle GPU with our implementation and scaling by the number of seeds.
A single iterative MNIST unlearning run (full-batch Adam, $500$ steps) takes a few seconds, and the Newton step is a single closed-form update. One (forget set, seed) cell of Table~\ref{tab:regime_compare}, comprising the $\thetaret$ solve, the leverage geometry, and all five unlearners, takes about one minute, so the regime comparison ($330$ cells, $100\times3$ random and $10\times3$ class) totals $\sim$$5$ GPU-hours. The matched-CE experiments of Figure~\ref{fig:loss_not_behavior} and Table~\ref{tab:loss_not_behavior} ($10$ seeds, $\rho$ bisection for six methods) take $\sim$$0.2$ GPU-hours. The per-example validations of Appendix~\ref{app:exp_extra} take about $4$ (MNIST) and $2$ (CIFAR-10) minutes per seed. The fixed-point validation of Table~\ref{tab:headline_main} takes $\sim$$2.9$ GPU-hours on MNIST and $\sim$$1.9$ on CIFAR-10 for $10$ seeds, dominated by the polishing of $\thetaret$ and the dense fp64 curvature solves.
The TOFU experiments train no base model. The oracle-proxy analysis of Appendix~\ref{app:exp_tofu_proxy} consumes the saved per-checkpoint per-example evaluations of one trajectory of the public \texttt{open-unlearning} recipe. The seed study of Appendix~\ref{app:exp_tofu_kff_krf} uses fifteen ten-epoch unlearning trajectories of that recipe with per-epoch evaluation ($6$--$49$ minutes each, $6.9$ GPU-hours in total), a one-time pass on Llama-3.2-1B-Instruct that caches the LoRA-space diagonal GGN over the $4{,}000$ TOFU training sequences ($\sim$$7$ minutes), and $K_{ff}$/$K_{rf}$ scoring that takes under two minutes per split. We do not separately enumerate compute used by preliminary or discarded experiments. We plan to open-source our code at \url{https://github.com/xingjianz/loss-is-not-behavior}; the code is not yet publicly available. \emph{Asset versions and licenses.} MNIST \citep{lecun2010mnist} (CC BY-SA 3.0), CIFAR-10 \citep{krizhevsky2009learning} together with the pretrained ResNet-18 released with the NeurIPS 2023 Machine Unlearning Challenge starting kit (\url{https://github.com/unlearning-challenge/starting-kit}), the TOFU dataset \citep{maini2024tofu} (MIT), the \texttt{open-unlearning} library \citep{openunlearning2025} (MIT), Llama-3.2-1B-Instruct \citep{grattafiori2024llama} (Llama 3.2 Community License), PyTorch (BSD-3-Clause), and Hugging Face Transformers/PEFT (Apache-2.0); all are used within their licensed terms.
Random seeds are explicit per run; all stochastic data subsampling, weight initialization, and SGD/Adam mini-batching is reproducible from the seed.

\paragraph{Data: MNIST.} 
We use the standard MNIST dataset of $28 \times 28$ grayscale digits ($60\,000$ training, $10\,000$ test), accessed locally as a NumPy archive. Per seed, we draw a random training subsample of $n_{\mathrm{tr}} = 4000$ images and a random test subsample of $n_{\mathrm{te}} = 2000$ images, normalize pixels to $[0,1]$, and standardize features per-coordinate using the training-subset mean and standard deviation (an additive $10^{-6}$ floor on the std). Labels are one-hot encoded for squared/CE losses.

\paragraph{Model and base training.} The MLP has architecture $784 \to 24 \to 10$ with $\tanh$ hidden activation, a softmax output head, and biases in both layers ($d=19{,}090$). We initialize weights from a zero-mean Gaussian with variance equal to inverse fan-in and biases to zero, using the run seed. Base models are trained for $40$ epochs with SGD (learning rate $0.1$, batch size $256$, momentum $0.9$), without L-BFGS polishing. The matched-CE motivation experiments use $\lambda=10^{-3}$; the regime comparison and fixed-point validation use $\lambda=10^{-2}$. The initial warm retrain runs for $20$ epochs from $\thetafull$ on the retain set, with learning rate $0.05$, batch size $64$, momentum $0.9$, and mini-batch seed $\mathrm{seed}+1$. The regime comparison replaces this reference and the fixed-point validation polishes it as specified below.

\paragraph{Unlearners and optimization.} The motivation and regime experiments start unlearning from $\thetafull$; the fixed-point experiments start from the polished $\thetaret$. The iterative methods optimize full-batch objectives. GradDiff minimizes retain CE minus $\rho$ times forget CE plus $\ell_2$ regularization. NPO uses reference log-probabilities from $\thetafull$ with $\beta=0.1$; SimNPO uses the reference-free loss with $\beta=2.5$ and margin $\delta=0$. The $\beta$ values follow the open-unlearning library defaults. Pure GA drops the retain term and ascends forget CE with only the $\ell_2$ anchor. IMU uses the same Pure-GA objective but multiplies each forget CE by a fixed weight equal to the norm of the final-layer CE gradient at $\thetafull$, normalized to mean one; this is the experimental proxy for the ideal influence norm in Table~\ref{tab:unified_gaps}. ProjGD applies the GradDiff direction after projecting out the span of $64$ retain-anchor CE gradients in the Euclidean metric, with the surrogate subspace recomputed along the trajectory. These choices differ from the closed-form theory rows, which solve stationary fixed-weight or frozen-local surrogate objectives.

The matched-CE protocols use deterministic full-batch SGD-style runners during $\rho$ bisection: $120$ steps per evaluation, lr $0.05$ for GradDiff/NPO/SimNPO/ProjGD and lr $0.005$ for Pure GA/IMU. The four iterative methods in Section~\ref{sec:exp_regime} instead use full-batch Adam ($\beta_1=0.9$, $\beta_2=0.999$, $\epsilon=10^{-8}$, lr $5\times10^{-3}$) for $500$ steps with $\lambda=10^{-2}$; no convergence-based stopping is used. Per-method $\rho$-bracket centers for matched-CE tuning were calibrated once at the start of the project as $\rho_0 \in \{$GradDiff: $5$, NPO: $5$, SimNPO: $30$, ProjGD: $5$, Pure GA: $2$, IMU: $0.5\}$. The Newton step is a single closed-form second-order update from $\thetafull$ using the local GGN curvature; it has no $\rho$ knob. The MNIST per-coordinate logit RMS is $\sqrt{(1/|A|)\sum_{i \in A}\|f(\boldsymbol{\theta},x_i) - f(\thetaret,x_i)\|_2^2 / c}$ with $c = 10$; TOFU instead uses per-example sequence-NLL deviation from the retain95 oracle.

\paragraph{Class-forget protocol (Figure~\ref{fig:loss_not_behavior}).} The forget set is the entire instance of a single class in the seed's training subsample (rotated across seeds: seed $i \to$ class $i \bmod 10$, so $|\Df|$ varies between $\sim 350$ and $\sim 490$ depending on which class lands in that seed's $4000$-sample subset). For each (seed, method) pair we tune $\rho$ by geometric bisection on the bracket $[\rho_0/16,\, 16\rho_0]$ until the achieved forget CE on $\Df$ falls within $\pm 10\%$ of $\mathrm{CE}(\thetaret,\Df)$, capped at $12$ bisection steps. Because Newton step has no $\rho$ knob, we report it without the matching constraint and flag it with hatched bars in Figure~\ref{fig:loss_not_behavior}. Figure~\ref{fig:loss_not_behavior} displays four observables: forget CE, weight distance $\|\theta-\thetaret\|_2$, KL$(p_{\thetaret}\|p_\theta)$ on $\Df$, and top-1 forget accuracy. The full observable set --- additionally including top-2 forget accuracy, confidence margin, held-out test accuracy, and a balanced-accuracy MIA based on a logistic attack on per-example loss (anchor batch from the retain set, equal forget/retain balance) --- appears in the random-forget aggregate Table~\ref{tab:loss_not_behavior}. Means and standard deviations are over $10$ seeds. The random-forget protocol uses the same recipe but with $\Df$ a uniformly random subsample of size $200$ from the seed's training subsample.

\input{tables/table4_matched_xe.tex}

\paragraph{Stability-fraction protocol (Section~\ref{sec:exp_regime}).} For each forget set, we form the reference $\thetaret$ by $500$ full-batch Adam steps on the full-data-normalized $\Robj$ from $\thetafull$, at learning rate $5\times10^{-3}$ and $\lambda=10^{-2}$. This fixed-step reference is not polished to a prescribed gradient tolerance. Curvature is estimated at this reference as
\[
\widetilde{\boldsymbol{H}}_r = \frac{1}{n-m}\sum_{i\in\Dr}\boldsymbol{J}_i^\top\boldsymbol{B}_i\boldsymbol{J}_i+\lambda\boldsymbol{I}.
\]
This retain-normalized GGN is a surrogate for the full-$n$ curvature in the theory. The $\boldsymbol{K}$ quantities in Table~\ref{tab:regime_compare} denote estimates $\boldsymbol{J}_a\widetilde{\boldsymbol{H}}_r^{-1}\boldsymbol{J}_f^\top$: the self-leverage block uses all $m$ forget examples, and the cross-leverage block uses $200$ retain anchors and all $m$ forget examples. We set $\rho=0.3\,n/\lambda_{\max}(\boldsymbol{K}_{ff})$ for GradDiff, NPO, SimNPO, and ProjGD as a heuristic motivated by the linear stability rule. Each takes $500$ full-batch Adam steps from $\thetafull$; Newton Step takes one update. The class regime uses $m=350$--$490$ (mean ${\approx}400$); the random regime uses $m=200$. Aggregates are over $10\times3=30$ class runs and $100\times3=300$ random runs. Output deviations are per-coordinate logit RMS on $\Df$ and $\Dr$. Mean $\rho$ is $1.14\times10^{-3}$ (random) and $6.69\times10^{-4}$ (class), varying by an order of magnitude across random forget sets. Across all five methods, mean forget CE ranges from $0.329$ to $0.359$ in the random regime. Class-regime CE is $4.368$--$4.404$ for the four iterative methods and $2.972$ for Newton Step.

\paragraph{GGN stability rule.}\label{app:stability_rule}
For unprojected GradDiff and frozen-weight (Sim)NPO GGN models formed with the same full-$n$ normalization as the retain objective, the softmax-curvature bound $\|\boldsymbol{B}_j^-\|_2\le\tfrac12$ bounds the effective kernel eigenvalues by $(w_{\max}/2)\lambda_{\max}(\boldsymbol{K}^{\mathrm{NL}}_{ff})$, where $w_{\max}$ is the largest sample weight. It equals $1$ for GradDiff and is at most $1$ under the table's $\sigma$ convention or $2$ under the deployed $2\sigma$ convention, equivalent after rescaling $\rho$. Thus $s<1$ suffices for positive definiteness of these frozen GGN models under the unweighted rule; locality up to $s=2$ remains an empirical observation. Table~\ref{tab:regime_compare} uses the differently normalized surrogate above, so its $s$ budget is not a direct certificate for the implemented objectives. Saturated SimNPO weights barely move the iterate under this rule, motivating the weighted kernel in Section~\ref{sec:exp_fixedpoint}.

\paragraph{CIFAR-10 over frozen features (Table~\ref{tab:headline_main} and Table~\ref{tab:headline_cifar_full}).} For CIFAR-10 \citep{krizhevsky2009learning} we use the ResNet-18 pretrained on CIFAR-10 that was released with the NeurIPS 2023 Machine Unlearning Challenge starting kit (test accuracy $88.6\%$) and extract its frozen $512$-dimensional penultimate features. The backbone is held fixed and is not the target of unlearning. The unlearning analysis runs on a 2-layer MLP head of architecture $512\to24\to10$ ($\tanh$ hidden activation, softmax output, $d=12{,}562$ parameters), so the feature vector plays the role of the input $\boldsymbol{x}$ and the MNIST pipeline applies unchanged. Per seed we draw $4000$ training and $2000$ test examples, standardize the features per coordinate, train the head to obtain $\thetafull$, and form $\thetaret$ exactly as for MNIST, with $|\Df|=200$ in the random regime and one full class ($|\Df|\approx400$) in the class regime.

\paragraph{Fixed-point validation protocol (Section~\ref{sec:exp_fixedpoint}).} Per dataset, regime, and seed, the reference $\thetaret$ is the $20$-epoch warm-start retrain above, polished to a stationary point of $\Robj$ under the normalization of~\eqref{eq:retain_obj} with $\lambda=10^{-2}$. The SGD retrain alone is not stationary, so we polish it with multi-restart L-BFGS, which reduces $\|\nabla\Robj\|_2$ below $10^{-7}$ and brings it to the stationary point in the basin of $\thetafull$. The retain curvature is the exact dense fp64 GGN $\boldsymbol{H}_r=\tfrac1n\sum_{i\in\Dr}\boldsymbol{J}_i^\top\boldsymbol{B}_i\boldsymbol{J}_i+\lambda\boldsymbol{I}$ at $\thetaret$. The stability rule sets $\rho=s\,n/\lambda_{\max}(\boldsymbol{K}_{ff})$ with $s=0.1$. For SimNPO the method-weighted kernel $\boldsymbol{W}^{1/2}\boldsymbol{K}_{ff}\boldsymbol{W}^{1/2}$ replaces $\boldsymbol{K}_{ff}$, because under the uniform rule its saturated class-regime weights ($\bar w\approx10^{-4}$) leave the endpoint at $\thetaret$. For GradDiff, NPO, SimNPO, Pure GA, and IMU, stationary endpoints are located by damped Newton continuation from $\thetaret$; NPO and SimNPO optimize their deployed objectives with parameter-dependent weights. ProjGD is run as projected-gradient dynamics from $\thetaret$ ($1000$ steps, learning rate $0.005$, $64$ retain anchors). The branch through $\thetaret$ is not the only stationary point near $\thetafull$. For MNIST seed $0$, a GradDiff run started at $\thetafull$ and converged with L-BFGS at the same $\rho$ ends $1.8$ (random) and $2.5$ (class) from the continuation endpoint in parameter norm, at a stationary point whose forget-logit RMS from $\thetaret$ ($0.26$ and $0.39$) the local predictor does not describe ($P^K_f\approx0$), while the same run started at $\thetaret$ returns to the continuation endpoint within $5\times10^{-5}$. The formulas describe the branch through the specified reference, and a deployed run inherits them only insofar as it converges to that branch. GradDiff, NPO, and SimNPO predictions use their rows of Table~\ref{tab:nonlinear_unified}, with NPO/SimNPO weights frozen at $\thetaret$ and their derivatives omitted. ProjGD is compared with the unprojected predictor~\eqref{eq:nonlinear_gap}; Pure GA/IMU are evaluated with simplified $K$- and Euclidean-Gram controls that omit the shifted-residual and reference-offset terms of their theoretical rows. Predictions are compared with the actual logit gap on $\Df$ and on $256$ random retain anchors by per-logit Pearson correlation and relative $\ell_2$ error. The Euclidean control replaces $\boldsymbol{H}_r^{-1}$ by $\lambda^{-1}\boldsymbol{I}$ in the prediction; for NPO/SimNPO it uses uniform forget weights, so these comparisons change both geometry and weighting. The shuffled-anchor control permutes the anchor rows of the prediction (it gives $|P|\le0.20$ on every row). The output KL quoted in Section~\ref{sec:exp_fixedpoint} is the mean per-example $\mathrm{KL}(p_{\thetaret}\Vert p_{\boldsymbol{\theta}})$ at the endpoint. Aggregates are over $10$ seeds per dataset and regime.

\paragraph{TOFU/Llama setup (Appendix~\ref{app:exp_tofu}).} 
All TOFU runs use the \texttt{open-unlearning} full-finetune recipe \citep{openunlearning2025} on TOFU \citep{maini2024tofu} with Llama-3.2-1B-Instruct \citep{grattafiori2024llama} as the base model (learning rate $10^{-5}$, 10 epochs, per-epoch checkpoints, library-default $\beta$ for NPO and SimNPO). The released retain95 (retain99) model is the oracle $\thetaret$ for forget05 (forget01). Two analyses are reported.

The \emph{oracle-NLL forget-channel proxy} (Appendix~\ref{app:exp_tofu_proxy}) uses one GradDiff forget05 trajectory and the score $s_{\mathrm{NLL}}(i) = \mathrm{avg\_loss}_{\thetaret}(x_i)$, the per-token NLL of the gold answer under the retain95 oracle. This score is not an explicit Hessian leverage estimate; it is a scalable oracle-evaluable surrogate for the forget residual term that appears together with $K_{ff}$ in Eq.~\eqref{eq:nonlinear_gap}. We sort all $200$ forget05 examples by $s_{\mathrm{NLL}}$, split them into quartiles, and compute the per-bin mean $|\Delta_i|$ at every saved checkpoint, where $\Delta_i$ is the per-example sequence-NLL deviation of the unlearned model from the oracle.

The \emph{seed study with $\thetafull$-only scores} (Appendix~\ref{app:exp_tofu_kff_krf}) covers $\{$forget01, forget05$\}\times\{$GradDiff, NPO$\}$ and forget05 SimNPO with three optimizer seeds each. The three seeds are the trainer seeds $0$, $1$, $2$ of the recipe with no other change, and the seed-$0$ GradDiff forget05 trajectory is the one used in Appendix~\ref{app:exp_tofu_proxy}. The scores are computed once at $\thetafull$ and never see $\thetaret$ or the trajectories. We restrict the parameter space to a LoRA adapter ($r{=}8$, target modules \texttt{all-linear}, $5.6{\times}10^6$ trainable parameters) and approximate the curvature by $\widehat{\boldsymbol{H}} = \mathrm{diag}\bigl(\tfrac{1}{N}\,\widehat{\mathrm{GGN}}\bigr) + \lambda \boldsymbol{I}$ with $\lambda=0.01$, where the diagonal uses one Monte-Carlo label sample per training sequence on $\Df \cup \Dr$, followed by a Woodbury low-rank correction built from $S{=}4$ MC pseudo-gradients on $\Df$. With $\boldsymbol{g}_i$ the LoRA-chart gradient of the answer NLL of example $i$ at $\thetafull$, the forget score is $\boldsymbol{g}_j^\top\widehat{\boldsymbol{H}}^{-1}\boldsymbol{g}_j$ on the $200$ (forget05) or $40$ (forget01) forget examples, and the retain score is $\sum_{j\in\Df}|\boldsymbol{g}_i^\top\widehat{\boldsymbol{H}}^{-1}\boldsymbol{g}_j|$ on the $400$ probes of the standard TOFU \texttt{retain\_perturbed} split. Each score is compared with the per-example $|\mathrm{NLL}_u-\mathrm{NLL}_{\mathrm{ret}}|$ at every per-epoch checkpoint except checkpoint~0, by Spearman correlation, AUC, and the Q4/Q1 ratio of mean gaps. Checkpoints are aggregated by the median within a trajectory and then averaged across seeds, with hierarchical-bootstrap confidence intervals ($B{=}1{,}000$, outer level trajectories, inner level examples) and within-checkpoint permutation nulls ($1{,}000$ permutations). Cheap baselines are the NLL at $\thetafull$ (forget) and the gradient norm $\|\boldsymbol{g}_i\|_2$ or the NLL at $\thetafull$ (retain). Table~\ref{tab:tofu_direction} re-analyzes the same trajectories against the NLL movement from $\thetafull$ with cosine, Pearson, and a fitted scale, and repeats the scoring on a half-frozen LoRA chart ($\boldsymbol{A}$ frozen, exact GGN over the $\boldsymbol{B}$ coordinates).

\section{LLM-scale diagnostics on TOFU/Llama}
\label{app:exp_tofu}

We ask whether the $K$-channel quantities of Section~\ref{sec:experiments} remain informative in the LLM regime, where explicit GGN leverage on the full parameter space is intractable. All runs use the \texttt{open-unlearning} recipe on TOFU with Llama-3.2-1B-Instruct \citep{openunlearning2025,maini2024tofu,grattafiori2024llama} and the released retain models as $\thetaret$. Appendix~\ref{app:exp_tofu_proxy} reports an oracle-NLL proxy for the forget-side residual channel along one GradDiff trajectory. Appendix~\ref{app:exp_tofu_kff_krf} reports $\thetafull$-only LoRA-space scores for the forget and retain channels against five method/split combinations with three optimizer seeds each. Both protocols are described in Appendix~\ref{app:exp_setup}.

\subsection{Oracle-NLL forget-channel proxy}
\label{app:exp_tofu_proxy}

Equation~\eqref{eq:nonlinear_gap} couples the forget self-leverage with the residual term $\boldsymbol{\delta}_f^-$; for next-token CE, the per-token NLL of the gold answer under $\thetaret$ is not a direct estimate of $K_{ff}$, but it is an oracle-evaluable proxy for this residual-weighted forget channel. We compute $s_{\mathrm{NLL}}(i) = \mathrm{avg\_loss}_{\thetaret}(x_i)$ for each forget example, sort examples by it, and split them into four equal-sized quartile groups. Because this score uses only forget-example oracle NLL, it probes the forget-side channel rather than retain cross-leverage.

\input{tables/table_tofu_corners.tex}

\textbf{Observation.} The qualitative prediction holds for this proxy: the high-score quartile (Q4) separates from the low-score one (Q1) throughout the trajectory, with a Q4/Q1 ratio of $1.3\text{--}3.1\times$ at intermediate checkpoints and a per-example Spearman$(s_{\mathrm{NLL}},|\Delta|)$ that starts at $0.99$ at step $0$ and decays to ${\sim} 0.1$ as the forget channel saturates (Table~\ref{tab:tofu_corners}). The diagnostic remains informative at LLM scale despite the absence of an explicit Hessian, optimization noise, and the move from squared loss to next-token CE. We note that this proxy still requires the retrained reference $\thetaret$; a fully deployable LLM-scale variant would build on the precomputable $\boldsymbol{e}_f^+$ form of Theorem~\ref{thm:retrain_linear_gap} (computed only at $\thetafull$), which we discuss as future work in Section~\ref{sec:limitations}.

\subsection{$\thetafull$-only $K_{ff}$ and $K_{rf}$ scores over independent seeds}
\label{app:exp_tofu_kff_krf}

The Section~\ref{sec:exp_regime} regime comparison hinges on $K_{ff}$ and $K_{rf}$ being computable. To check that this is not specific to small networks, and to remove the dependence on a single trajectory and on the oracle, we compute both channels once at $\thetafull$ in a LoRA chart ($r{=}8$, $5.6{\times}10^6$ coordinates) with a diagonal-plus-Woodbury GGN, and compare them with the realized per-example gaps $|\mathrm{NLL}_u-\mathrm{NLL}_{\mathrm{ret}}|$ of GradDiff, NPO, and SimNPO over three optimizer seeds on forget01 and forget05 (full protocol in Appendix~\ref{app:exp_setup}). The forget score is the residual-weighted self-leverage $\boldsymbol{g}_j^\top\widehat{\boldsymbol{H}}^{-1}\boldsymbol{g}_j$ and the retain score is the cross-leverage row sum $\sum_{j\in\Df}|\boldsymbol{g}_i^\top\widehat{\boldsymbol{H}}^{-1}\boldsymbol{g}_j|$, both from gradients at $\thetafull$ alone, so this is the precomputable form anticipated in Section~\ref{sec:limitations}.

\input{tables/table_tofu_seeds.tex}

\textbf{Observation.} The $\thetafull$-only scores carry usable per-example signal at LLM scale. The retain score ranks the retain probes that the unlearners disturb most, with Spearman $0.26$--$0.49$ and Q4/Q1 gap ratios of $1.3$--$3.7$, far outside the permutation null ($\pm0.02$), on both splits and for all three methods (SimNPO on forget05 only), across independent seeds. The forget score ranks the forget examples whose treatment departs most from retraining, with Spearman $0.14$--$0.29$ and Q4/Q1 ratios of $1.1$--$1.8$, and is at the null only for NPO on forget01 ($p=0.99$). Both scores are available before any unlearning run, so they can flag which examples will be hard to unlearn faithfully and where an evaluation of the unlearned model should look. Two limits are visible in the same table. The scores are comparable to, not clearly better than, the per-example loss or gradient norm at $\thetafull$ (margins between $-0.17$ and $+0.06$), and they are not calibrated predictions of retraining-gap magnitude. A separate diagnostic in Table~\ref{tab:tofu_direction} evaluates the profile and scale of NLL movement from $\thetafull$ on the same trajectories. The per-example profile of the NLL movement away from $\thetafull$ follows the scores (cosine above its permutation null on every row, centered Pearson $0.16$--$0.58$ forget and $0.19$--$0.39$ retain), but the relative error after fitting one scale stays at $0.6$--$0.8$ and the fitted scale varies by $27$--$65\%$ across checkpoints and seeds. Several differences from the local validation may contribute. The curvature is a diagonal-plus-rank-$4$ GGN on a LoRA chart of $5.6\times10^6$ of the $1.2\times10^9$ parameters, the endpoints are early-stopped ten-epoch runs, and these scores target sequence NLL rather than the logit gaps of Section~\ref{sec:exp_fixedpoint}. These endpoints sit $0.07$--$1.2$ nats of retain KL from a cold-start $\thetaret$, far beyond the $\le10^{-5}$-nat regime of the positive local validation in Section~\ref{sec:exp_fixedpoint}. Computing the scores on a half-frozen LoRA chart with an exact GGN on the active half changes little. These comparisons do not isolate the effects of chart restriction, curvature approximation, endpoint nonstationarity, or target choice; whether richer curvature or more local endpoints improve prediction remains open. The output-space KL of the final checkpoints ($0.07$--$1.2$ nats on the retain probes) confirms that these recipes do not leave the retain channel untouched.

\input{tables/table_tofu_direction.tex}

\section{Additional Experimental Validations}
\label{app:exp_extra}

These additional experiments test the framework along axes complementary to Section~\ref{sec:experiments} and are referenced from there but not relied on for the main claims.

\paragraph{ProjGD and Newton-step behavior in the regime comparison.}\label{app:regime_method_discussion}
ProjGD has nearly unchanged retain RMS across regimes ($0.438\to0.440$) but larger output gaps than GradDiff, NPO, and SimNPO. Its projection removes components along sampled retain-loss gradients; this does not guarantee unchanged retain logits, and the regime comparison does not isolate the cause of its larger gaps.
The Newton step, which recovers $\thetaret$ exactly in the linear case (Table~\ref{tab:unified_gaps}), shows the largest nonlinear output gaps ($0.46\to0.73$ on the forget set). This illustrates the limits of the local quadratic approximation. The comparison does not separately identify GGN mismatch, higher-order error, or imperfect reference stationarity.

\paragraph{Linear closed-form sanity check (Theorem~\ref{thm:linear_gap}).} On a synthetic ridge regression problem we compare the analytical gap of Eq.~\eqref{eq:linear_gap} to the numerical solution of the GradDiff stationarity equations. The relative error is $<10^{-10}$ across all evaluation sets and all values of $\rho$ inside the stability region, confirming that the closed form is exact in the linear setting.

\paragraph{Full local validation and controls (Eq.~\eqref{eq:nonlinear_gap}).} Tables~\ref{tab:headline_mnist_full} and~\ref{tab:headline_cifar_full} expand the local GradDiff/NPO/SimNPO validation and the ProjGD and Pure GA/IMU comparisons of Table~\ref{tab:headline_main}, under the protocol of Section~\ref{sec:exp_fixedpoint}. Panel~A adds the held-out test channel $\boldsymbol{K}_{af}$ ($P^K_a=0.80$--$0.96$ on MNIST and $0.71$--$0.95$ on CIFAR-10 for GradDiff, NPO, and ProjGD), the relative $\ell_2$ error of the predicted gap vector, the output-space KL to $\thetaret$, and the IMU rows, which behave like Pure GA. For Pure GA and IMU the correlations are small or negative (between $-0.51$ and $0.32$) under both the $K$-geometry and the Euclidean-Gram predictor, and their endpoints sit $0.7$--$3$ nats of KL away from $\thetaret$, so these rows are outside the local regime as discussed in Section~\ref{sec:limitations}. Panel~B sweeps the removal scale $s\in\{0.05,0.2,0.5,1,2\}$. $P^K_f$ stays at or above $0.96$ (MNIST) and $0.88$ (CIFAR-10) up to $s=2$, while the endpoint displacement from $\thetaret$ grows by more than an order of magnitude and the class-regime forget cross-entropy rises. Table~\ref{tab:beta_swap} is the $\beta$-swap control referenced in Section~\ref{sec:exp_regime}.

\input{tables/table_headline_mnist_full.tex}
\input{tables/table_headline_cifar_full.tex}
\input{tables/table_beta_swap.tex}

\paragraph{Stability divergence as $s \to 1$ (Corollary~\ref{lem:stability}).} Sweeping the stability fraction $s = \tfrac{\rho}{n}\lambda_{\max}(\boldsymbol{K}_{ff}) \in \{0.05, 0.1, 0.2, 0.4, 0.6, 0.8, 0.9, 0.95\}$, the predicted and actual logit deviations increase as the inverse factor approaches the stability boundary. The norm bound of Eq.~\eqref{eq:linear_norm_bound} has the correct denominator but is extremely loose in absolute magnitude, as expected from a worst-case operator-norm argument.

\clearpage
\input{checklist.tex}

\end{document}

%% file: tables/table_regime_compare.tex
\begin{table}[t]
\centering
\footnotesize
\setlength{\tabcolsep}{6pt}
\begin{tabular}{l l c c c}
\toprule
\multicolumn{2}{l}{\emph{Forget-set geometry under $\thetaret$}} & $\lambda_{\max}(\boldsymbol{K}_{ff})$ & $\|\boldsymbol{K}_{rf}\|_F$ & $\|\boldsymbol{e}_f^-\|$ \\
\midrule
\multicolumn{2}{l}{Random ($300$ runs)} & $1.45\!\times\!10^{6}$ & $4.74\!\times\!10^{5}$ & $5.28$ \\
\multicolumn{2}{l}{Class ($30$ runs)} & $2.27\!\times\!10^{6}$ & $6.26\!\times\!10^{5}$ & $24.25$ \\
\midrule
\multicolumn{5}{l}{\emph{Per-method output-space deviation from $\thetaret$ (mean $\pm$ std)}} \\
Method & Regime & RMS$_f$ & RMS$_r$ & RMS$_f$/RMS$_r$ \\
\midrule
GradDiff & random & $0.010{\scriptstyle\,\pm\,0.016}$ & $0.009{\scriptstyle\,\pm\,0.016}$ & $1.12$ \\
GradDiff & class & $0.038{\scriptstyle\,\pm\,0.037}$ & $0.031{\scriptstyle\,\pm\,0.029}$ & $1.23$ \\
\addlinespace[2pt]
NPO & random & $0.010{\scriptstyle\,\pm\,0.015}$ & $0.009{\scriptstyle\,\pm\,0.015}$ & $1.12$ \\
NPO & class & $0.032{\scriptstyle\,\pm\,0.032}$ & $0.025{\scriptstyle\,\pm\,0.022}$ & $1.25$ \\
\addlinespace[2pt]
SimNPO & random & $0.005{\scriptstyle\,\pm\,0.009}$ & $0.004{\scriptstyle\,\pm\,0.008}$ & $1.05$ \\
SimNPO & class & $0.009{\scriptstyle\,\pm\,0.011}$ & $0.009{\scriptstyle\,\pm\,0.011}$ & $1.01$ \\
\addlinespace[2pt]
ProjGD & random & $0.449{\scriptstyle\,\pm\,0.036}$ & $0.438{\scriptstyle\,\pm\,0.031}$ & $1.03$ \\
ProjGD & class & $0.509{\scriptstyle\,\pm\,0.065}$ & $0.440{\scriptstyle\,\pm\,0.034}$ & $1.16$ \\
\addlinespace[2pt]
Newton Step & random & $0.455{\scriptstyle\,\pm\,0.037}$ & $0.416{\scriptstyle\,\pm\,0.030}$ & $1.10$ \\
Newton Step & class & $0.728{\scriptstyle\,\pm\,0.074}$ & $0.470{\scriptstyle\,\pm\,0.040}$ & $1.55$ \\
\addlinespace[2pt]
\bottomrule
\end{tabular}
\caption{\textbf{Class vs.\ random forgetting on MNIST.} Top: mean leverage estimates using retain-normalized GGN curvature at $\thetaret$. Bottom: per-coordinate logit RMS gaps from $\thetaret$ on $\Df$ and $\Dr$ and their ratio. Class-regime forget accuracy is zero for GradDiff, NPO, SimNPO, and ProjGD, and $0.005\pm0.020$ for Newton Step. Loss ranges are summarized in the text; setup details in Appendix~\ref{app:exp_setup}.}
\label{tab:regime_compare}
\end{table}

%% file: tables/table_headline_grid.tex
\begin{table}[t]
\centering
\scriptsize
\setlength{\tabcolsep}{2.2pt}
\renewcommand{\arraystretch}{0.8}
\setlength{\abovecaptionskip}{4pt}
\begin{tabular}{l l cccc cccc}
\toprule
 & & \multicolumn{4}{c}{Random forget set} & \multicolumn{4}{c}{Class forget set} \\
\cmidrule(lr){3-6} \cmidrule(lr){7-10}
Data & Method & $P^K_f$ & $P^G_f$ & $P^K_r$ & $P^G_r$ & $P^K_f$ & $P^G_f$ & $P^K_r$ & $P^G_r$ \\
\midrule
MNIST & GradDiff & 0.98$\pm$0.02 & 0.83$\pm$0.06 & 0.72$\pm$0.10 & 0.40$\pm$0.06 & 0.98$\pm$0.00 & 0.81$\pm$0.04 & 0.95$\pm$0.02 & 0.59$\pm$0.09 \\
 & NPO & 0.98$\pm$0.02 & 0.83$\pm$0.06 & 0.72$\pm$0.10 & 0.40$\pm$0.05 & 0.98$\pm$0.00 & 0.81$\pm$0.04 & 0.95$\pm$0.02 & 0.58$\pm$0.09 \\
 & SimNPO & 0.94$\pm$0.12 & 0.52$\pm$0.22 & 0.74$\pm$0.09 & 0.24$\pm$0.05 & 0.98$\pm$0.01 & -0.03$\pm$0.13 & 0.97$\pm$0.01 & 0.00$\pm$0.18 \\
 & ProjGD & 0.96$\pm$0.02 & 0.94$\pm$0.02 & 0.75$\pm$0.01 & 0.83$\pm$0.03 & 0.95$\pm$0.01 & 0.97$\pm$0.01 & 0.87$\pm$0.03 & 0.87$\pm$0.05 \\
 & Pure GA & 0.01$\pm$0.01 & 0.02$\pm$0.02 & 0.05$\pm$0.03 & 0.04$\pm$0.04 & -0.36$\pm$0.06 & -0.39$\pm$0.05 & -0.11$\pm$0.03 & -0.23$\pm$0.07 \\
\midrule
CIFAR-10 & GradDiff & 0.89$\pm$0.03 & 0.74$\pm$0.04 & 0.82$\pm$0.05 & 0.63$\pm$0.06 & 0.97$\pm$0.01 & 0.94$\pm$0.02 & 0.87$\pm$0.09 & 0.74$\pm$0.07 \\
 & NPO & 0.89$\pm$0.03 & 0.74$\pm$0.04 & 0.82$\pm$0.05 & 0.63$\pm$0.06 & 0.97$\pm$0.01 & 0.94$\pm$0.02 & 0.87$\pm$0.09 & 0.74$\pm$0.07 \\
 & SimNPO & 0.91$\pm$0.01 & 0.60$\pm$0.11 & 0.86$\pm$0.01 & 0.51$\pm$0.09 & 0.97$\pm$0.01 & -0.14$\pm$0.17 & 0.86$\pm$0.10 & 0.04$\pm$0.17 \\
 & ProjGD & 0.79$\pm$0.03 & 0.74$\pm$0.05 & 0.71$\pm$0.04 & 0.67$\pm$0.06 & 0.97$\pm$0.01 & 0.98$\pm$0.01 & 0.87$\pm$0.02 & 0.80$\pm$0.04 \\
 & Pure GA & 0.22$\pm$0.07 & 0.26$\pm$0.07 & 0.29$\pm$0.08 & 0.27$\pm$0.08 & -0.51$\pm$0.07 & -0.45$\pm$0.08 & -0.04$\pm$0.01 & -0.23$\pm$0.06 \\
\bottomrule
\end{tabular}
\caption{\textbf{Local gap validation and controls.} Per-logit Pearson correlation between predicted and actual gaps on $\Df$ ($f$) and $256$ retain anchors ($r$), using retain-GGN ($P^K$) or Euclidean-Gram ($P^G$) geometry; mean$\pm$sd over 10 seeds at $s=0.1$. GradDiff/NPO/SimNPO test the local predictors of Table~\ref{tab:nonlinear_unified}; NPO/SimNPO predictors freeze reference weights, and SimNPO uses the weighted $\rho$ rule. ProjGD compares its projected-gradient endpoint with the unprojected predictor; Pure GA is a nonlocal control.}
\label{tab:headline_main}
\end{table}

%% file: tables/table4_matched_xe.tex
\begin{table}[t]
\centering\small
\resizebox{\textwidth}{!}{%
\begin{tabular}{l ccccccc}
\toprule
        Method & Forget CE & Top-1 $\downarrow$ & Top-2 $\downarrow$ & KL$(p_{r^*}\|p)$ $\downarrow$ & Margin & MIA $\to 0.5$ & Test acc $\uparrow$ \\
\midrule
        GradDiff & $0.172\!\pm\!0.037$ & $0.961\!\pm\!0.009$ & $0.988\!\pm\!0.009$ & $0.057\!\pm\!0.014$ & $0.930\!\pm\!0.007$ & $0.497\!\pm\!0.027$ & $0.904\!\pm\!0.006$ \\
        NPO & $0.174\!\pm\!0.036$ & $0.959\!\pm\!0.008$ & $0.988\!\pm\!0.008$ & $0.054\!\pm\!0.012$ & $0.928\!\pm\!0.008$ & $0.499\!\pm\!0.026$ & $0.904\!\pm\!0.006$ \\
        SimNPO & $0.176\!\pm\!0.041$ & $0.904\!\pm\!0.024$ & $0.995\!\pm\!0.006$ & $0.053\!\pm\!0.014$ & $0.872\!\pm\!0.011$ & $0.490\!\pm\!0.029$ & $0.904\!\pm\!0.007$ \\
        PureGA & $0.171\!\pm\!0.036$ & $0.958\!\pm\!0.007$ & $0.988\!\pm\!0.009$ & $0.075\!\pm\!0.019$ & $0.924\!\pm\!0.008$ & $0.495\!\pm\!0.028$ & $0.902\!\pm\!0.007$ \\
        IMU & $0.178\!\pm\!0.039$ & $0.972\!\pm\!0.006$ & $0.986\!\pm\!0.005$ & $0.106\!\pm\!0.028$ & $0.948\!\pm\!0.008$ & $0.507\!\pm\!0.027$ & $0.901\!\pm\!0.007$ \\
        ProjGD & $0.171\!\pm\!0.040$ & $0.961\!\pm\!0.009$ & $0.986\!\pm\!0.010$ & $0.054\!\pm\!0.013$ & $0.933\!\pm\!0.007$ & $0.496\!\pm\!0.025$ & $0.905\!\pm\!0.007$ \\
        Newton Step & $0.409\!\pm\!0.057$ & $0.880\!\pm\!0.019$ & $0.952\!\pm\!0.013$ & $0.077\!\pm\!0.022$ & $0.869\!\pm\!0.015$ & $0.476\!\pm\!0.024$ & $0.905\!\pm\!0.007$ \\
        Retraining & $0.173\!\pm\!0.034$ & $0.940\!\pm\!0.012$ & $0.985\!\pm\!0.009$ & $0$ & $0.901\!\pm\!0.012$ & $0.489\!\pm\!0.026$ & $0.907\!\pm\!0.007$ \\
\bottomrule
\end{tabular}%
}
\caption{Random-forget aggregate results on MNIST/MLP ($|\Df|=200$, $10$ seeds). The tunable gradient methods are matched to the retrained forget cross-entropy; Newton Step has no $\rho$ knob and is reported as its single second-order step. Setup in Appendix~\ref{app:exp_setup}; class-forget version in Figure~\ref{fig:loss_not_behavior}.}
\label{tab:loss_not_behavior}
\end{table}

%% file: tables/table_tofu_corners.tex
\begin{table}[t]
\centering
\small
\begin{tabular}{l c c c c c c}
\toprule
 & \multicolumn{4}{c}{Forget $|\Delta_i|$ by oracle-NLL proxy quartile} & & \\
\cmidrule(lr){2-5}
GradDiff step & Q1 (low) & Q2 & Q3 & Q4 (high) & Q4/Q1 & Spearman$(s_{\mathrm{NLL}},|\Delta|)$ \\
\midrule
        $0$ & $1.219$ & $2.037$ & $2.740$ & $3.518$ & $2.89\times$ & $+0.993$ \\
        $12$ & $1.006$ & $1.757$ & $2.386$ & $3.129$ & $3.11\times$ & $+0.909$ \\
        $31$ & $0.903$ & $1.099$ & $1.154$ & $1.202$ & $1.33\times$ & $+0.251$ \\
        $50$ & $1.228$ & $1.589$ & $1.578$ & $1.309$ & $1.07\times$ & $+0.104$ \\
        $60$ & $1.250$ & $1.622$ & $1.613$ & $1.316$ & $1.05\times$ & $+0.091$ \\
\bottomrule
\end{tabular}
\caption{Per-example forget deviation $|\Delta_i|$ grouped by oracle-NLL proxy quartile along the GradDiff trajectory on TOFU/Llama-3.2-1B-Instruct (forget05, $|\Df|=200$, 50 examples per quartile). Q1/Q4 contain the lowest/highest proxy scores. Setup in Appendix~\ref{app:exp_setup}. The rows are successive checkpoints of one trajectory, not independent runs. The step-0 Spearman is close to a tautology, because at step~0 $\Delta_i=\mathrm{NLL}_{\thetaret}(x_i)-\mathrm{NLL}_{\thetafull}(x_i)$ while the proxy is $\mathrm{NLL}_{\thetaret}(x_i)$ and $\mathrm{NLL}_{\thetafull}(x_i)$ is small (mean $0.13$ vs.\ $2.5$). Independent seeds and $\thetafull$-only scores are in Tables~\ref{tab:tofu_seeds} and~\ref{tab:tofu_direction}.}
\label{tab:tofu_corners}
\end{table}

%% file: tables/table_tofu_seeds.tex
\begin{table}[htbp]
\centering
\scriptsize
\setlength{\tabcolsep}{4pt}
\resizebox{\textwidth}{!}{%
\begin{tabular}{llccccc}
\toprule
\multicolumn{7}{l}{\textit{Forget channel}}\\
\midrule
Split & Method & Spearman$_f$ [95\% CI] / AUC$_f$ & Q4/Q1$_f$ [CI] & $\Delta_f$ vs.\ loss & null$_f$ ($p$) & token KL$_f$ \\
\midrule
forget01 & GradDiff & 0.14 [4.7e-3, 0.24] / 0.61 & 1.48 [1.05, 1.75] & 0.04 & [-0.07, 0.07] (0.001) & 3.49$\pm$0.09 \\
forget05 & GradDiff & 0.29 [0.25, 0.33] / 0.68 & 1.83 [1.61, 1.97] & -9.0e-3 & [-0.03, 0.03] (0.001) & 5.32$\pm$0.12 \\
forget01 & NPO & 2.4e-3 [-0.07, 0.13] / 0.52 & 1.12 [0.95, 1.31] & 0.03 & [-0.07, 0.07] (0.988) & 1.69$\pm$0.08 \\
forget05 & NPO & 0.21 [0.16, 0.23] / 0.58 & 1.39 [1.29, 1.48] & 0.06 & [-0.03, 0.03] (0.001) & 1.37$\pm$0.10 \\
forget05 & SimNPO & 0.23 [0.20, 0.26] / 0.61 & 1.28 [1.25, 1.34] & 0.04 & [-0.03, 0.03] (0.001) & 2.29$\pm$0.03 \\
\midrule
\multicolumn{7}{l}{\textit{Retain channel}}\\
\midrule
Split & Method & Spearman$_r$ [95\% CI] / AUC$_r$ & Q4/Q1$_r$ [CI] & $\Delta_r$ vs.\ baseline & null$_r$ ($p$) & token KL$_r$ \\
\midrule
forget01 & GradDiff & 0.41 [0.40, 0.44] / 0.72 & 2.92 [2.23, 4.28] & -0.08 & [-0.02, 0.02] (0.001) & 0.694$\pm$0.042 \\
forget05 & GradDiff & 0.28 [0.26, 0.31] / 0.63 & 1.60 [1.55, 1.76] & -0.07\,(loss) & [-0.02, 0.02] (0.001) & 1.183$\pm$0.049 \\
forget01 & NPO & 0.49 [0.46, 0.53] / 0.73 & 3.69 [3.28, 4.22] & -0.08 & [-0.02, 0.02] (0.001) & 0.142$\pm$2.4e-4 \\
forget05 & NPO & 0.26 [0.26, 0.31] / 0.62 & 1.32 [1.31, 1.41] & 0.03\,(loss) & [-0.02, 0.02] (0.001) & 0.780$\pm$0.014 \\
forget05 & SimNPO & 0.32 [0.30, 0.35] / 0.69 & 2.62 [2.43, 2.78] & -0.17\,(loss) & [-0.02, 0.02] (0.001) & 0.065$\pm$0.001 \\
\bottomrule
\end{tabular}%
}
\caption{\textbf{TOFU/Llama-3.2-1B: $\thetafull$-only $K$-channel scores vs.\ realized per-example gaps over independent optimizer seeds.} Unlearning uses the open-unlearning full-finetune recipe (lr $10^{-5}$, 10 epochs, 8--10 per-epoch checkpoints per trajectory) for $\{$forget01, forget05$\}\times\{$GradDiff, NPO$\}$ and forget05 SimNPO, with 3 optimizer seeds each. The reference is the released retain99/95 model. Scores are computed once at $\thetafull$ in the LoRA chart ($r{=}8$, diagonal GGN $+$ MC-Woodbury $\widehat{\boldsymbol{H}}$, $\lambda{=}0.01$) without access to $\thetaret$ or to the trajectories. With $\boldsymbol{g}_i$ the LoRA-chart gradient of the answer NLL of example $i$ at $\thetafull$, the forget score is $\boldsymbol{g}_j^\top\widehat{\boldsymbol{H}}^{-1}\boldsymbol{g}_j$ and the retain score is $\sum_{j\in\Df}|\boldsymbol{g}_i^\top\widehat{\boldsymbol{H}}^{-1}\boldsymbol{g}_j|$. Actual gaps are the per-example $|\mathrm{NLL}_u-\mathrm{NLL}_{\mathrm{ret}}|$ on the forget set (top) and on 400 retain anchors (bottom). Per-epoch checkpoints (checkpoint~0 excluded) are aggregated by the median within a trajectory and then averaged across seeds, with hierarchical-bootstrap CIs ($B{=}1{,}000$). The $\Delta$ columns subtract the Spearman of a cheap per-example baseline, namely the NLL at $\thetafull$ on the forget channel and the gradient norm (or the NLL where marked ``loss'') on the retain channel. The null columns give the 95\% interval of the within-checkpoint permutation null ($1{,}000$ permutations) and the two-sided $p$ of the observed Spearman. Token KL is the mean answer-token $\mathrm{KL}(p_{\thetaret}\Vert p_{\boldsymbol{\theta}_u})$ at the final checkpoint. The scores identify examples with the largest NLL deviations from retraining on both channels, above the permutation null in every cell except the NPO forget01 forget channel, and are comparable to the cheap per-example baselines.}
\label{tab:tofu_seeds}
\end{table}

%% file: tables/table_tofu_direction.tex
\begin{table}[htbp]
\centering
\scriptsize
\setlength{\tabcolsep}{4pt}
\resizebox{\textwidth}{!}{%
\begin{tabular}{llccccc}
\toprule
Split & Method & $P$ $f/r$ & Cosine $f/r$ & perm.\ null $f/r$ & RelErr $f/r$ & scale CV (within/pooled) \\
\midrule
forget01 & GradDiff & 0.52$\pm$0.07 / 0.33$\pm$0.03 & 0.78$\pm$0.02 / 0.72$\pm$0.06 & 0.59$\pm$0.01 / 0.57$\pm$6.2e-3 & 0.63$\pm$0.03 / 0.69$\pm$0.06 & 0.51 / 0.65 \\
forget05 & GradDiff & 0.58$\pm$0.03 / 0.23$\pm$5.6e-3 & 0.75$\pm$0.01 / 0.72$\pm$3.5e-3 & 0.52$\pm$6.3e-3 / 0.64$\pm$4.2e-3 & 0.67$\pm$0.02 / 0.69$\pm$3.6e-3 & 0.50 / 0.51 \\
forget01 & NPO & 0.40$\pm$0.09 / 0.27$\pm$0.03 & 0.74$\pm$0.03 / 0.70$\pm$0.02 & 0.62$\pm$0.01 / 0.60$\pm$5.0e-3 & 0.67$\pm$0.03 / 0.71$\pm$0.02 & 0.39 / 0.49 \\
forget05 & NPO & 0.16$\pm$0.03 / 0.19$\pm$0.01 & 0.61$\pm$6.6e-3 / 0.75$\pm$1.3e-3 & 0.57$\pm$4.3e-3 / 0.70$\pm$2.9e-3 & 0.79$\pm$5.1e-3 / 0.67$\pm$1.5e-3 & 0.43 / 0.43 \\
forget05 & SimNPO & 0.43$\pm$0.01 / 0.25$\pm$0.02 & 0.62$\pm$0.02 / 0.58$\pm$0.03 & 0.33$\pm$0.01 / 0.44$\pm$6.0e-3 & 0.78$\pm$0.01 / 0.82$\pm$0.02 & 0.27 / 0.28 \\
\midrule
\multicolumn{7}{l}{\textit{Dropout-LoRA score arm: $A$ frozen, exact GGN on the active $B$ half, vocab-dimension MC Woodbury}} \\
\midrule
forget01 & GradDiff & 0.55$\pm$0.08 / 0.39$\pm$0.02 & 0.78$\pm$0.03 / 0.76$\pm$0.05 & 0.57$\pm$0.01 / 0.60$\pm$5.9e-3 & 0.63$\pm$0.04 / 0.64$\pm$0.06 & 0.49 / 0.64 \\
forget05 & GradDiff & 0.57$\pm$0.03 / 0.35$\pm$4.2e-3 & 0.74$\pm$0.01 / 0.83$\pm$1.3e-3 & 0.52$\pm$6.4e-3 / 0.73$\pm$3.3e-3 & 0.67$\pm$0.02 / 0.56$\pm$2.0e-3 & 0.50 / 0.51 \\
forget01 & NPO & 0.33$\pm$0.17 / 0.31$\pm$0.04 & 0.71$\pm$0.05 / 0.74$\pm$0.02 & 0.60$\pm$0.01 / 0.62$\pm$4.7e-3 & 0.70$\pm$0.05 / 0.68$\pm$0.03 & 0.37 / 0.46 \\
forget05 & NPO & 0.16$\pm$0.03 / 0.23$\pm$0.02 & 0.61$\pm$6.2e-3 / 0.84$\pm$1.4e-3 & 0.56$\pm$4.3e-3 / 0.80$\pm$2.3e-3 & 0.80$\pm$4.7e-3 / 0.54$\pm$2.2e-3 & 0.43 / 0.43 \\
forget05 & SimNPO & 0.44$\pm$0.01 / 0.39$\pm$0.01 & 0.63$\pm$0.02 / 0.67$\pm$0.02 & 0.33$\pm$9.9e-3 / 0.51$\pm$4.9e-3 & 0.78$\pm$0.01 / 0.74$\pm$0.02 & 0.27 / 0.27 \\
\bottomrule
\end{tabular}%
}
\caption{\textbf{TOFU/Llama-3.2-1B: per-example profile of the NLL movement from $\thetafull$ against the $\thetafull$-only scores.} Predictor-only re-analysis of the trajectories of Table~\ref{tab:tofu_seeds} (3 seeds, 7--10 checkpoints each) with the same forget and retain scores, which carry no method weights $\boldsymbol{W}^-$ and no $\rho/n$ factor. The actual quantity is the per-example sequence-NLL movement $\mathrm{nll}(\mathrm{ckpt})-\mathrm{nll}(\thetafull)$ (absolute value on the retain channel). This is the movement away from the original model and not the gap to retraining, so the test asks only whether the per-example profile of that movement is predicted up to one scale. $P$ is the per-example Pearson correlation (its permutation null is $0$). Cosine is the cosine between the two vectors (median over checkpoints per seed, mean$\pm$sd over seeds) and perm.\ null is its within-checkpoint permutation baseline (500 permutations), so the profile is predicted only where Cosine exceeds its null. RelErr is $\min_s\|s\,k-\Delta\|_2/\|\Delta\|_2$ after fitting the single free scale, and scale CV is the coefficient of variation of that fitted scale (within-seed / pooled). The second block computes the scores on the half-frozen LoRA chart ($\boldsymbol{A}$ frozen, exact GGN over the active $\boldsymbol{B}$ coordinates, $S{=}4$ vocabulary-dimension MC pseudo-gradients). NPO/SimNPO rows use the unweighted scores and are therefore conservative checks, not calibrated $\boldsymbol{W}^-$ predictions.}
\label{tab:tofu_direction}
\end{table}

%% file: tables/table_headline_mnist_full.tex
\begin{table}[htbp]
\centering
\scriptsize
\setlength{\tabcolsep}{3pt}
\textbf{Panel A: agreement at the operating point $s=0.1$}\\[2pt]
\resizebox{\textwidth}{!}{%
\begin{tabular}{llccccccc}
\toprule
Regime & Method &
$P^K_f$ & $P^G_f$ & $P^K_r$ & $P^G_r$ & $P^K_a$ &
RelErr $f/r$ &
KL$(\thetaret\Vert\boldsymbol{\theta}_u)$ $f/r$ \\
\midrule
Random & GradDiff & 0.98$\pm$0.02 & 0.83$\pm$0.06 & 0.72$\pm$0.10 & 0.40$\pm$0.06 & 0.87$\pm$0.05 & 0.30$\pm$0.17 / 0.77$\pm$0.11 & 5.3e-7$\pm$4.7e-7 / 9.2e-10$\pm$1.1e-9 \\
Random & NPO$^\dagger$ & 0.98$\pm$0.02 & 0.83$\pm$0.06 & 0.72$\pm$0.10 & 0.40$\pm$0.05 & 0.87$\pm$0.06$^\ddagger$ & 0.31$\pm$0.17 / 0.77$\pm$0.11 & 4.8e-7$\pm$4.4e-7 / 7.9e-10$\pm$9.6e-10 \\
Random & SimNPO$^\dagger$ & 0.94$\pm$0.12 & 0.52$\pm$0.22 & 0.74$\pm$0.09 & 0.24$\pm$0.05 & \blank & 0.31$\pm$0.21 / 0.75$\pm$0.10 & 1.5e-7$\pm$2.2e-7 / 1.8e-10$\pm$2.1e-10 \\
Random & ProjGD & 0.96$\pm$0.02 & 0.94$\pm$0.02 & 0.75$\pm$0.01 & 0.83$\pm$0.03 & 0.80$\pm$0.04 & 13.76$\pm$2.63 / 6.00$\pm$0.53 & 3.5e-9$\pm$2.6e-9 / 6.5e-11$\pm$7.9e-11 \\
Random & Pure GA & 5.5e-3$\pm$0.01 & 0.02$\pm$0.02 & 0.05$\pm$0.03 & 0.04$\pm$0.04 & 0.04$\pm$0.02 & 1.00$\pm$4.9e-6 / 1.00$\pm$1.9e-6 & 1.82$\pm$0.03 / 1.90$\pm$0.02 \\
Random & IMU & 5.5e-3$\pm$0.01 & -0.04$\pm$0.02 & 0.05$\pm$0.03 & -0.05$\pm$0.03 & 0.04$\pm$0.02 & 1.00$\pm$4.9e-6 / 1.00$\pm$1.9e-6 & 1.82$\pm$0.03 / 1.90$\pm$0.02 \\
\midrule
Class & GradDiff & 0.98$\pm$4.9e-3 & 0.81$\pm$0.04 & 0.95$\pm$0.02 & 0.59$\pm$0.09 & 0.96$\pm$9.7e-3 & 0.31$\pm$0.07 / 0.31$\pm$0.05 & 7.9e-6$\pm$1.0e-5 / 7.3e-8$\pm$1.3e-7 \\
Class & NPO$^\dagger$ & 0.98$\pm$4.9e-3 & 0.81$\pm$0.04 & 0.95$\pm$0.02 & 0.58$\pm$0.09 & 0.96$\pm$9.4e-3$^\ddagger$ & 0.31$\pm$0.07 / 0.30$\pm$0.05 & 4.7e-6$\pm$6.1e-6 / 4.4e-8$\pm$7.5e-8 \\
Class & SimNPO$^\dagger$ & 0.98$\pm$6.2e-3 & -0.03$\pm$0.13 & 0.97$\pm$8.2e-3 & 2.5e-4$\pm$0.18 & \blank & 0.27$\pm$0.08 / 0.25$\pm$0.03 & 1.4e-6$\pm$1.9e-6 / 1.4e-8$\pm$1.7e-8 \\
Class & ProjGD & 0.95$\pm$0.01 & 0.97$\pm$0.01 & 0.87$\pm$0.03 & 0.87$\pm$0.05 & 0.92$\pm$0.01 & 8.43$\pm$1.37 / 8.28$\pm$0.75 & 2.7e-7$\pm$4.0e-7 / 2.8e-9$\pm$5.2e-9 \\
Class & Pure GA & -0.36$\pm$0.06 & -0.39$\pm$0.05 & -0.11$\pm$0.03 & -0.23$\pm$0.07 & -0.15$\pm$0.01 & 1.00$\pm$7.3e-4 / 1.00$\pm$3.9e-5 & 1.24$\pm$0.16 / 1.93$\pm$0.02 \\
Class & IMU & -0.36$\pm$0.06 & -0.39$\pm$0.05 & -0.11$\pm$0.03 & -0.23$\pm$0.06 & -0.15$\pm$0.01 & 1.00$\pm$7.3e-4 / 1.00$\pm$3.9e-5 & 1.24$\pm$0.16 / 1.93$\pm$0.02 \\
\bottomrule
\end{tabular}%
}

\vspace{0.6em}
\textbf{Panel B: agreement across the normalized removal scale $s$ (GradDiff / NPO), $\rho=s\,n/\lambda_{\max}(\boldsymbol{K}_{ff})$}\\[2pt]
\resizebox{\textwidth}{!}{%
\begin{tabular}{llcccccccc}
\toprule
& & \multicolumn{5}{c}{$P^K_f$ at $s$} & \multicolumn{2}{c}{at $s=2$} & \\
\cmidrule(lr){3-7}\cmidrule(lr){8-9}
Regime & Method & $0.05$ & $0.2$ & $0.5$ & $1.0$ & $2.0$ & $P^K_r$ & $\|\boldsymbol{\theta}_u-\thetaret\|_2$
& forget CE $s{=}0.05{\to}2$ \\
\midrule
Random & GradDiff & 0.98$\pm$0.02 & 0.98$\pm$0.01 & 0.98$\pm$0.02 & 0.98$\pm$0.02 & 0.97$\pm$0.03 & 0.73$\pm$0.10 & 0.016$\pm$0.009 & 0.32$\pm$0.03 $\to$ 0.32$\pm$0.03 \\
Random & NPO & 0.98$\pm$0.02 & 0.98$\pm$0.02 & 0.98$\pm$0.02 & 0.98$\pm$0.02 & 0.97$\pm$0.03 & 0.72$\pm$0.11 & 0.015$\pm$0.008 & 0.32$\pm$0.03 $\to$ 0.32$\pm$0.03 \\
\midrule
Class & GradDiff & 0.98$\pm$4.9e-3 & 0.98$\pm$5.0e-3 & 0.97$\pm$6.2e-3 & 0.96$\pm$0.02 & 0.97$\pm$0.01 & 0.92$\pm$0.07 & 0.093$\pm$0.063 & 4.41$\pm$0.25 $\to$ 4.65$\pm$0.36 \\
Class & NPO & 0.98$\pm$5.1e-3 & 0.97$\pm$5.2e-3 & 0.97$\pm$0.02 & 0.97$\pm$8.7e-3 & 0.96$\pm$0.01 & 0.91$\pm$0.08 & 0.070$\pm$0.043 & 4.41$\pm$0.25 $\to$ 4.59$\pm$0.31 \\
\bottomrule
\end{tabular}%
}
\caption{\textbf{MNIST: full local gap validation and controls.}
The setting is a 2-layer $\tanh$ MLP ($784\to24\to10$, $d{=}19{,}090$), $n{=}4000$, $\lambda{=}10^{-2}$, exact fp64 throughout.
Protocol (Appendix~\ref{app:exp_setup}). The warm retrain is polished to stationarity under the full-$n$ normalization and $\boldsymbol{H}_r$ is the exact fp64 GGN. GradDiff, NPO, SimNPO, Pure GA, and IMU endpoints are tracked from $\thetaret$ by Newton continuation. ProjGD uses projected-gradient dynamics and the unprojected predictor. GradDiff/NPO/SimNPO test the local predictors of Table~\ref{tab:nonlinear_unified}; Pure GA and IMU are controls outside the local regime. Mean$\pm$sd over 10 seeds.
At $s=0.1$ the GradDiff endpoint lies $0.0008$ (random) / $0.0040$ (class) from $\thetaret$ in parameter norm, versus $\|\thetafull-\thetaret\|_2=3.39$ / $4.06$, illustrating the locality of the retain-anchored validation.
Columns are as in Table~\ref{tab:headline_main}. $P^K_a$ is the per-logit Pearson correlation on the held-out test subsample ($\boldsymbol{K}_{af}$ channel), RelErr is the relative $\ell_2$ error of the predicted gap vector, and KL is the output-space $\mathrm{KL}(p_{\thetaret}\Vert p_{\boldsymbol{\theta}_u})$ at the endpoint. The shuffled-anchor control gives $|P|\le0.20$.
$^\dagger$method-specific $(\boldsymbol{K},\boldsymbol{W}^-)$ scoring with frozen reference weights for NPO/SimNPO. $^\ddagger$test channel scored with uniform weights (for NPO the weights at $\thetaret$ are nearly uniform), not stored for SimNPO.
Panel~B repeats the construction at larger removal scales. $P^K_f$ stays at or above $0.96$ up to $s=2$, while the endpoint moves farther from $\thetaret$ and the class-regime forget cross-entropy rises. Here $s$ is normalized by $\lambda_{\max}(\boldsymbol{K}_{ff})$ alone, so $s=1$ is the squared-loss boundary of Corollary~\ref{lem:stability}. Under cross-entropy the softmax curvature $\boldsymbol{B}_f^-$ enters the stability factor of Eq.~\eqref{eq:nonlinear_gap}, and the tracked fixed point continues past $s=1$.}
\label{tab:headline_mnist_full}
\end{table}

%% file: tables/table_headline_cifar_full.tex
\begin{table}[htbp]
\centering
\scriptsize
\setlength{\tabcolsep}{3pt}
\textbf{Panel A: agreement at the operating point $s=0.1$}\\[2pt]
\resizebox{\textwidth}{!}{%
\begin{tabular}{llccccccc}
\toprule
Regime & Method &
$P^K_f$ & $P^G_f$ & $P^K_r$ & $P^G_r$ & $P^K_a$ &
RelErr $f/r$ &
KL$(\thetaret\Vert\boldsymbol{\theta}_u)$ $f/r$ \\
\midrule
Random & GradDiff & 0.89$\pm$0.03 & 0.74$\pm$0.04 & 0.82$\pm$0.05 & 0.63$\pm$0.06 & 0.83$\pm$0.04 & 0.78$\pm$0.03 / 0.84$\pm$0.06 & 7.1e-8$\pm$3.8e-8 / 4.9e-9$\pm$3.4e-9 \\
Random & NPO$^\dagger$ & 0.89$\pm$0.03 & 0.74$\pm$0.04 & 0.82$\pm$0.05 & 0.63$\pm$0.06 & 0.83$\pm$0.04$^\ddagger$ & 0.78$\pm$0.03 / 0.84$\pm$0.06 & 6.8e-8$\pm$3.6e-8 / 4.7e-9$\pm$3.1e-9 \\
Random & SimNPO$^\dagger$ & 0.91$\pm$8.6e-3 & 0.60$\pm$0.11 & 0.86$\pm$0.01 & 0.51$\pm$0.09 & \blank & 0.68$\pm$0.03 / 0.69$\pm$0.05 & 1.1e-8$\pm$4.0e-9 / 1.5e-9$\pm$3.0e-10 \\
Random & ProjGD & 0.79$\pm$0.03 & 0.74$\pm$0.05 & 0.71$\pm$0.04 & 0.67$\pm$0.06 & 0.71$\pm$0.03 & 8.29$\pm$1.20 / 6.82$\pm$0.90 & 1.6e-9$\pm$8.5e-10 / 2.3e-10$\pm$1.2e-10 \\
Random & Pure GA & 0.22$\pm$0.07 & 0.26$\pm$0.07 & 0.29$\pm$0.08 & 0.27$\pm$0.08 & 0.24$\pm$0.07 & 1.00$\pm$9.5e-6 / 1.00$\pm$7.8e-6 & 2.72$\pm$1.07 / 2.71$\pm$1.06 \\
Random & IMU & 0.24$\pm$0.06 & 0.07$\pm$0.04 & 0.32$\pm$0.08 & 0.07$\pm$0.05 & 0.26$\pm$0.06 & 1.00$\pm$6.0e-6 / 1.00$\pm$5.3e-6 & 2.94$\pm$0.25 / 2.91$\pm$0.26 \\
\midrule
Class & GradDiff & 0.97$\pm$0.01 & 0.94$\pm$0.02 & 0.87$\pm$0.09 & 0.74$\pm$0.07 & 0.94$\pm$0.04 & 0.78$\pm$0.14 / 0.61$\pm$0.16 & 3.0e-6$\pm$1.4e-6 / 3.2e-9$\pm$2.0e-9 \\
Class & NPO$^\dagger$ & 0.97$\pm$0.01 & 0.94$\pm$0.02 & 0.87$\pm$0.09 & 0.74$\pm$0.07 & 0.94$\pm$0.04$^\ddagger$ & 0.78$\pm$0.14 / 0.61$\pm$0.16 & 2.1e-6$\pm$9.6e-7 / 2.2e-9$\pm$1.4e-9 \\
Class & SimNPO$^\dagger$ & 0.97$\pm$0.01 & -0.14$\pm$0.17 & 0.86$\pm$0.10 & 0.04$\pm$0.17 & \blank & 0.79$\pm$0.15 / 0.63$\pm$0.17 & 1.4e-6$\pm$8.3e-7 / 1.6e-9$\pm$1.1e-9 \\
Class & ProjGD & 0.97$\pm$8.3e-3 & 0.98$\pm$8.8e-3 & 0.87$\pm$0.02 & 0.80$\pm$0.04 & 0.95$\pm$8.2e-3 & 13.82$\pm$1.19 / 11.48$\pm$0.62 & 4.5e-8$\pm$3.0e-8 / 5.2e-11$\pm$3.8e-11 \\
Class & Pure GA & -0.51$\pm$0.07 & -0.45$\pm$0.08 & -0.04$\pm$0.01 & -0.23$\pm$0.06 & -0.12$\pm$8.3e-3 & 1.00$\pm$2.6e-4 / 1.00$\pm$4.6e-6 & 0.74$\pm$0.14 / 2.11$\pm$0.01 \\
Class & IMU & -0.51$\pm$0.07 & -0.44$\pm$0.08 & -0.04$\pm$0.01 & -0.23$\pm$0.06 & -0.12$\pm$8.3e-3 & 1.00$\pm$2.6e-4 / 1.00$\pm$4.6e-6 & 0.74$\pm$0.14 / 2.11$\pm$0.01 \\
\bottomrule
\end{tabular}%
}

\vspace{0.6em}
\textbf{Panel B: agreement across the normalized removal scale $s$ (GradDiff / NPO), $\rho=s\,n/\lambda_{\max}(\boldsymbol{K}_{ff})$}\\[2pt]
\resizebox{\textwidth}{!}{%
\begin{tabular}{llcccccccc}
\toprule
& & \multicolumn{5}{c}{$P^K_f$ at $s$} & \multicolumn{2}{c}{at $s=2$} & \\
\cmidrule(lr){3-7}\cmidrule(lr){8-9}
Regime & Method & $0.05$ & $0.2$ & $0.5$ & $1.0$ & $2.0$ & $P^K_r$ & $\|\boldsymbol{\theta}_u-\thetaret\|_2$
& forget CE $s{=}0.05{\to}2$ \\
\midrule
Random & GradDiff & 0.89$\pm$0.03 & 0.89$\pm$0.03 & 0.89$\pm$0.03 & 0.89$\pm$0.03 & 0.88$\pm$0.04 & 0.80$\pm$0.05 & 0.023$\pm$0.007 & 0.05$\pm$0.01 $\to$ 0.05$\pm$0.01 \\
Random & NPO & 0.89$\pm$0.03 & 0.89$\pm$0.03 & 0.89$\pm$0.03 & 0.89$\pm$0.03 & 0.88$\pm$0.05 & 0.80$\pm$0.05 & 0.023$\pm$0.007 & 0.05$\pm$0.01 $\to$ 0.05$\pm$0.01 \\
\midrule
Class & GradDiff & 0.97$\pm$0.01 & 0.97$\pm$0.02 & 0.97$\pm$0.03 & 0.95$\pm$0.08 & 0.97$\pm$0.01 & 0.85$\pm$0.07 & 0.030$\pm$0.004 & 3.31$\pm$0.18 $\to$ 3.39$\pm$0.20 \\
Class & NPO & 0.97$\pm$0.01 & 0.97$\pm$0.02 & 0.97$\pm$0.02 & 0.96$\pm$0.05 & 0.97$\pm$0.01 & 0.86$\pm$0.07 & 0.025$\pm$0.003 & 3.31$\pm$0.18 $\to$ 3.38$\pm$0.19 \\
\bottomrule
\end{tabular}%
}
\caption{\textbf{CIFAR-10: full local gap validation and controls.}
The setting is a 2-layer $\tanh$ MLP head ($512\to24\to10$, $d{=}12{,}562$) over frozen ResNet-18 penultimate features, $n{=}4000$, $\lambda{=}10^{-2}$, exact fp64 throughout.
Protocol (Appendix~\ref{app:exp_setup}). The warm retrain is polished to stationarity under the full-$n$ normalization and $\boldsymbol{H}_r$ is the exact fp64 GGN. GradDiff, NPO, SimNPO, Pure GA, and IMU endpoints are tracked from $\thetaret$ by Newton continuation. ProjGD uses projected-gradient dynamics and the unprojected predictor. GradDiff/NPO/SimNPO test the local predictors of Table~\ref{tab:nonlinear_unified}; Pure GA and IMU are controls outside the local regime. Mean$\pm$sd over 10 seeds.
At $s=0.1$ the GradDiff endpoint lies $0.0011$ (random) / $0.0015$ (class) from $\thetaret$ in parameter norm, versus $\|\thetafull-\thetaret\|_2=0.91$ / $1.70$, illustrating the locality of the retain-anchored validation.
Columns are as in Table~\ref{tab:headline_main}. $P^K_a$ is the per-logit Pearson correlation on the held-out test subsample ($\boldsymbol{K}_{af}$ channel), RelErr is the relative $\ell_2$ error of the predicted gap vector, and KL is the output-space $\mathrm{KL}(p_{\thetaret}\Vert p_{\boldsymbol{\theta}_u})$ at the endpoint. The shuffled-anchor control gives $|P|\le0.18$.
$^\dagger$method-specific $(\boldsymbol{K},\boldsymbol{W}^-)$ scoring with frozen reference weights for NPO/SimNPO. $^\ddagger$test channel scored with uniform weights (for NPO the weights at $\thetaret$ are nearly uniform), not stored for SimNPO.
Panel~B repeats the construction at larger removal scales. $P^K_f$ stays at or above $0.88$ up to $s=2$, while the endpoint moves farther from $\thetaret$ and the class-regime forget cross-entropy rises. Here $s$ is normalized by $\lambda_{\max}(\boldsymbol{K}_{ff})$ alone, so $s=1$ is the squared-loss boundary of Corollary~\ref{lem:stability}. Under cross-entropy the softmax curvature $\boldsymbol{B}_f^-$ enters the stability factor of Eq.~\eqref{eq:nonlinear_gap}, and the tracked fixed point continues past $s=1$.}
\label{tab:headline_cifar_full}
\end{table}

%% file: tables/table_beta_swap.tex
\begin{table}[htbp]
\centering
\scriptsize
\setlength{\tabcolsep}{3pt}
\resizebox{\textwidth}{!}{%
\begin{tabular}{llccccccc}
\toprule
Regime & Method & $\beta$ & $\bar w(\thetaret)$ & forget CE & $P^K_f$ & $P^K_r$ & Cosine $f/r$ & RelErr $f/r$ \\
\midrule
Random & NPO & 0.1 & 0.99$\pm$1.2e-3 & 0.32$\pm$0.03 & 0.98$\pm$0.02 & 0.72$\pm$0.10 & 0.98$\pm$0.02 / 0.72$\pm$0.10 & 0.41$\pm$0.21 / 0.82$\pm$0.12 \\
Random & NPO & 2.5 & 0.87$\pm$0.02 & 0.32$\pm$0.03 & 0.60$\pm$0.26 & 0.43$\pm$0.10 & 0.60$\pm$0.26 / 0.43$\pm$0.10 & 4.93$\pm$2.08 / 2.98$\pm$0.62 \\
Random & SimNPO & 0.1 & 0.98$\pm$1.5e-3 & 0.32$\pm$0.03 & 0.98$\pm$0.02 & 0.72$\pm$0.10 & 0.98$\pm$0.02 / 0.72$\pm$0.10 & 0.42$\pm$0.22 / 0.84$\pm$0.12 \\
Random & SimNPO & 2.5 & 0.78$\pm$0.02 & 0.32$\pm$0.03 & 0.58$\pm$0.28 & 0.39$\pm$0.10 & 0.58$\pm$0.28 / 0.39$\pm$0.10 & 5.82$\pm$2.48 / 4.26$\pm$0.73 \\
\midrule
Class & NPO & 0.1 & 0.79$\pm$0.01 & 4.41$\pm$0.25 & 0.98$\pm$5.2e-3 & 0.95$\pm$0.02 & 0.98$\pm$5.2e-3 / 0.95$\pm$0.02 & 0.60$\pm$0.11 / 0.45$\pm$0.06 \\
Class & NPO & 2.5 & 8.0e-4$\pm$1.6e-3 & 4.40$\pm$0.25 & 0.04$\pm$0.16 & 0.01$\pm$0.07 & 0.04$\pm$0.16 / 0.01$\pm$0.07 & 1444.73$\pm$1884.98 / 412.95$\pm$628.17 \\
Class & SimNPO & 0.1 & 0.78$\pm$0.01 & 4.41$\pm$0.25 & 0.98$\pm$5.2e-3 & 0.95$\pm$0.02 & 0.98$\pm$5.2e-3 / 0.95$\pm$0.02 & 0.61$\pm$0.11 / 0.45$\pm$0.06 \\
Class & SimNPO & 2.5 & 9.7e-5$\pm$4.8e-5 & 4.40$\pm$0.25 & 0.03$\pm$0.13 & 0.01$\pm$0.02 & 0.03$\pm$0.13 / 0.01$\pm$0.02 & 1500.59$\pm$1931.20 / 417.88$\pm$630.95 \\
\bottomrule
\end{tabular}%
}
\caption{\textbf{$\beta$-swap control on MNIST.} The four corners of the (method, $\beta$) square under the protocol of Table~\ref{tab:headline_mnist_full}, scored with uniform weights (10 seeds). $\bar w(\thetaret)$ is the mean per-example forget weight at the reference in the deployed $2\sigma(\cdot)$ convention, $2\sigma(-\beta(\log p^{\thetafull}-\log p^{\thetaret}))$ for NPO and $2\sigma(\beta\log p^{\thetaret})$ for SimNPO. At $\beta=0.1$ both methods agree with the predictor equally well. At $\beta=2.5$ both collapse in the class regime. There the weights saturate ($\bar w\sim10^{-4}$ when $\mathrm{CE}_f(\thetaret)\approx4.4$), the effective forget term vanishes, and the endpoint stays at $\thetaret$ up to optimizer noise, so the forget CE equals the reference value and RelErr $\gg1$. Swapping $\beta$ swaps the behavior, so the NPO/SimNPO difference comes from $\beta$-induced weight saturation and not from the reference-based vs.\ reference-free form.}
\label{tab:beta_swap}
\end{table}

%% file: checklist.tex
\section*{NeurIPS Paper Checklist}

\begin{enumerate}

\item {\bf Claims}
    \item[] Question: Do the main claims made in the abstract and introduction accurately reflect the paper's contributions and scope?
    \item[] Answer: \answerYes{} % Replace by \answerYes{}, \answerNo{}, or \answerNA{}.
    \item[] Justification: The main claims made in the abstract and introduction are that we provide a unified view of machine unlearning methods, and that we provide a theoretical analysis of the output gaps of these methods. These claims are supported by the content of the paper.

\item {\bf Limitations}
    \item[] Question: Does the paper discuss the limitations of the work performed by the authors?
    \item[] Answer: \answerYes{} % Replace by \answerYes{}, \answerNo{}, or \answerNA{}.
    \item[] Justification: The paper includes discussions of the limitations of the proposed methods, including assumptions made for theoretical analysis, practical approximation errors, and potential practicality issues.

\item {\bf Theory assumptions and proofs}
    \item[] Question: For each theoretical result, does the paper provide the full set of assumptions and a complete (and correct) proof?
    \item[] Answer: \answerYes{} % Replace by \answerYes{}, \answerNo{}, or \answerNA{}.
    \item[] Justification: The paper provides a complete set of assumptions and proofs for all theoretical results, either in the main text or in the supplementary material.

    \item {\bf Experimental result reproducibility}
    \item[] Question: Does the paper fully disclose all the information needed to reproduce the main experimental results of the paper to the extent that it affects the main claims and/or conclusions of the paper (regardless of whether the code and data are provided or not)?
    \item[] Answer: \answerYes{} % Replace by \answerYes{}, \answerNo{}, or \answerNA{}.
    \item[] Justification: Appendix~\ref{app:exp_setup} describes the datasets, model architectures, training procedures, and evaluation metrics used in the experiments.

\item {\bf Open access to data and code}
    \item[] Question: Does the paper provide open access to the data and code, with sufficient instructions to faithfully reproduce the main experimental results, as described in supplemental material?
    \item[] Answer: \answerNo{} % Replace by \answerYes{}, \answerNo{}, or \answerNA{}.
    \item[] Justification: We plan to open-source our code; it is not yet publicly available. The datasets are public, and Appendix~\ref{app:exp_setup} describes the experimental protocols.

\item {\bf Experimental setting/details}
    \item[] Question: Does the paper specify all the training and test details (e.g., data splits, hyperparameters, how they were chosen, type of optimizer) necessary to understand the results?
    \item[] Answer: \answerYes{} % Replace by \answerYes{}, \answerNo{}, or \answerNA{}.
    \item[] Justification: Appendix~\ref{app:exp_setup} describes the datasets, model architectures, training procedures, and evaluation metrics used in the experiments.

\item {\bf Experiment statistical significance}
    \item[] Question: Does the paper report error bars suitably and correctly defined or other appropriate information about the statistical significance of the experiments?
    \item[] Answer: \answerYes{} % Replace by \answerYes{}, \answerNo{}, or \answerNA{}.
    \item[] Justification: The paper reports results across multiple runs with error bars that capture variability across runs for the main experiments.

\item {\bf Experiments compute resources}
    \item[] Question: For each experiment, does the paper provide sufficient information on the computer resources (type of compute workers, memory, time of execution) needed to reproduce the experiments?
    \item[] Answer: \answerYes{} % Replace by \answerYes{}, \answerNo{}, or \answerNA{}.
    \item[] Justification: Appendix~\ref{app:exp_setup} reports the detailed resource requirements for all experiments.
    
\item {\bf Code of ethics}
    \item[] Question: Does the research conducted in the paper conform, in every respect, with the NeurIPS Code of Ethics \url{https://neurips.cc/public/EthicsGuidelines}?
    \item[] Answer: \answerYes{} % Replace by \answerYes{}, \answerNo{}, or \answerNA{}.
    \item[] Justification: The research conducted in the paper adheres to the NeurIPS Code of Ethics.

\item {\bf Broader impacts}
    \item[] Question: Does the paper discuss both potential positive societal impacts and negative societal impacts of the work performed?
    \item[] Answer: \answerYes{} % Replace by \answerYes{}, \answerNo{}, or \answerNA{}.
    \item[] Justification: The paper discusses potential societal impacts.
    
\item {\bf Safeguards}
    \item[] Question: Does the paper describe safeguards that have been put in place for responsible release of data or models that have a high risk for misuse (e.g., pre-trained language models, image generators, or scraped datasets)?
    \item[] Answer: \answerNA{} % Replace by \answerYes{}, \answerNo{}, or \answerNA{}.
    \item[] Justification: The paper poses no such risks.

\item {\bf Licenses for existing assets}
    \item[] Question: Are the creators or original owners of assets (e.g., code, data, models), used in the paper, properly credited and are the license and terms of use explicitly mentioned and properly respected?
    \item[] Answer: \answerYes{} % Replace by \answerYes{}, \answerNo{}, or \answerNA{}.
    \item[] Justification: The paper properly credits all creators and original owners of assets used, and the licenses and terms of use are explicitly mentioned and respected.

\item {\bf New assets}
    \item[] Question: Are new assets introduced in the paper well documented and is the documentation provided alongside the assets?
    \item[] Answer: \answerNA{} % Replace by \answerYes{}, \answerNo{}, or \answerNA{}.
    \item[] Justification: The paper does not release new assets.

\item {\bf Crowdsourcing and research with human subjects}
    \item[] Question: For crowdsourcing experiments and research with human subjects, does the paper include the full text of instructions given to participants and screenshots, if applicable, as well as details about compensation (if any)? 
    \item[] Answer: \answerNA{} % Replace by \answerYes{}, \answerNo{}, or \answerNA{}.
    \item[] Justification: The paper does not involve crowdsourcing nor research with human subjects.

\item {\bf Institutional review board (IRB) approvals or equivalent for research with human subjects}
    \item[] Question: Does the paper describe potential risks incurred by study participants, whether such risks were disclosed to the subjects, and whether Institutional Review Board (IRB) approvals (or an equivalent approval/review based on the requirements of your country or institution) were obtained?
    \item[] Answer: \answerNA{} % Replace by \answerYes{}, \answerNo{}, or \answerNA{}.
    \item[] Justification: The paper does not involve crowdsourcing nor research with human subjects.

\item {\bf Declaration of LLM usage}
    \item[] Question: Does the paper describe the usage of LLMs if it is an important, original, or non-standard component of the core methods in this research? Note that if the LLM is used only for writing, editing, or formatting purposes and does \emph{not} impact the core methodology, scientific rigor, or originality of the research, declaration is not required.
    %this research? 
    \item[] Answer: \answerNA{} % Replace by \answerYes{}, \answerNo{}, or \answerNA{}.
    \item[] Justification: The paper does not involve LLMs as any important, original, or non-standard components.

\end{enumerate}